\documentclass[final]{article} %
\usepackage{iclr2026_conference,times}
\usepackage{amsthm}
\usepackage{url}
\usepackage{enumitem}
\usepackage{bbm}
\usepackage{lipsum}
\usepackage{graphicx}
\usepackage[most]{tcolorbox}
\usepackage{soul}
\usepackage{caption}
\usepackage{subcaption}
\usepackage{wrapfig}
\usepackage{lipsum}
\usepackage{enumitem}
\usepackage{multirow}
\usepackage{tikz}
\usepackage{mathtools}
\usepackage{booktabs}
\usepackage{amsmath}
\usepackage{xcolor}
\usetikzlibrary{arrows.meta}
\usepackage{basic_commands}
\usepackage{algorithm}
\usepackage{tabularew}
\usepackage[endLComment=,italicComments=false]{algpseudocodex}

\AddToHook{cmd/ttfamily/after}{\frenchspacing}

\theoremstyle{plain}
\newtheorem{theorem}{Theorem}[section]

\newtheorem{lemma}[theorem]{Lemma}

\theoremstyle{definition}

\theoremstyle{remark}

\definecolor{myred}{HTML}{F54254}
\definecolor{myorange}{HTML}{FFB135}
\definecolor{mygreen}{HTML}{10BD35}
\definecolor{myblue}{HTML}{598BE7}
\definecolor{mypurple}{HTML}{9A1C6B}
\definecolor{plgray}{HTML}{999999}
\renewcommand{\tt}[1]{\texttt{#1}}
\renewcommand{\phi}{\varphi}

\newcommand{\pl}[1]{{\color{plgray} #1}}

\newcommand{\methodname}{TRL\xspace}

\hypersetup{urlcolor=myblue,citecolor=myblue,linkcolor=myblue}

\setlist[itemize]{itemsep=1pt, leftmargin=20pt}

\newcommand{\cutsectionup}{\vspace{-0pt}}
\newcommand{\cutsectiondown}{\vspace{-0pt}}

\title{Transitive RL: \\ Value Learning via Divide and Conquer}

\author{Seohong Park\thanks{Equal contribution.}\:\:$^{1}$ \quad Aditya Oberai\footnotemark[1]\:\:$^{1}$ \quad Pranav Atreya$^{1}$ \quad Sergey Levine$^{1}$ \\
$^{1}$University of California, Berkeley \\
\texttt{\{seohong, aoberai\}@berkeley.edu}
}

\iclrfinalcopy %
\begin{document}

\maketitle

\begin{abstract}
In this work, we present Transitive Reinforcement Learning (\methodname), a new value learning algorithm based on a divide-and-conquer paradigm.
\methodname is designed for offline goal-conditioned reinforcement learning (GCRL) problems,
where the aim is to find a policy that can reach any state from any other state in the smallest number of steps.
\methodname converts a triangle inequality structure present in GCRL into a practical divide-and-conquer value update rule.
This has several advantages compared to alternative value learning paradigms.
Compared to temporal difference (TD) methods, \methodname suffers less from bias accumulation, as in principle it only requires $O(\log T)$ recursions
(as opposed to $O(T)$ in TD learning) to handle a length-$T$ trajectory.
Unlike Monte Carlo methods, \methodname suffers less from high variance as it performs dynamic programming.
Experimentally, we show that \methodname achieves the best performance in highly challenging, long-horizon benchmark tasks
compared to previous offline GCRL algorithms.

Blog post: \url{https://seohong.me/blog/rl-without-td-learning}
\end{abstract}

\begin{figure}[h!]
    \centering
    \includegraphics[width=1.0\textwidth]{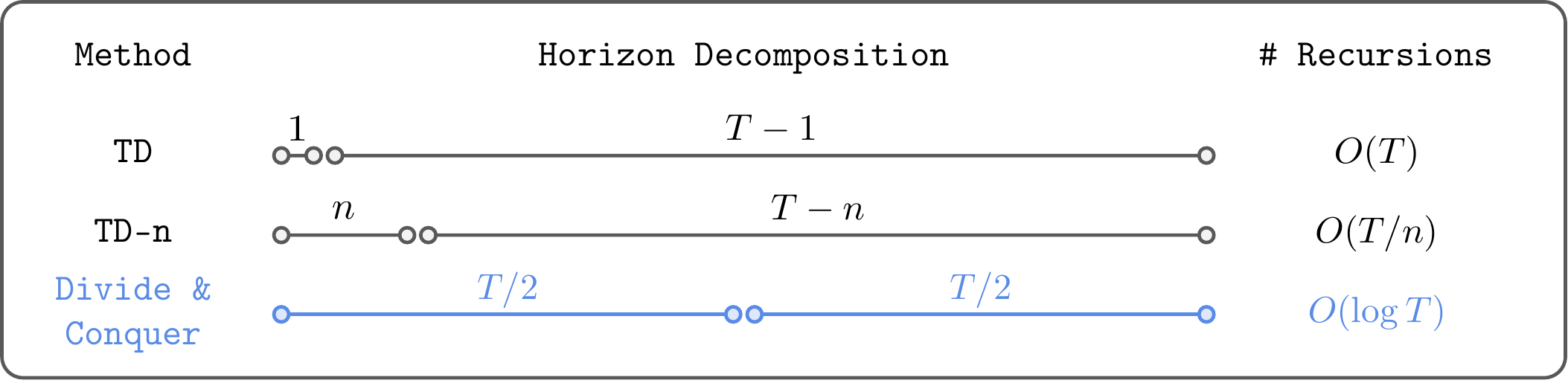}
    \caption{
    \footnotesize
    \textbf{Transitive RL.}
    Transitive RL is based on the divide-and-conquer paradigm,
    which can in theory reduce the number of Bellman recursions to $O(\log T)$ in the best case,
    unlike TD-based methods.
    }
    \label{fig:teaser}
\end{figure}

\cutsectionup
\section{Introduction}
\cutsectiondown
\label{sec:intro}

A fundamental challenge in off-policy reinforcement learning (RL) is the curse of horizon~\citep{horizon_liu2018}.
In off-policy RL, we typically train a value function using (some variant of) temporal difference (TD) learning with Bellman backups,
$Q(s, a) \gets r(s, a) + \gamma \max_{a'} \bar Q(s', a')$,
and train a policy to maximize the learned value function.
The problem is that,
each Bellman update involves regression toward a \emph{biased} target value,
and these biases \emph{accumulate} over the horizon through recursion.
This bias accumulation is one of the main obstacles hindering the scaling of off-policy RL to complex, long-horizon tasks~\citep{sharsa_park2025}.

A standard technique to mitigate this bias accumulation issue in practice
is to mix Monte Carlo (MC) returns,
either for short-horizon trajectory chunks (\eg, $n$-step returns)
or even for full trajectories (\eg, purely MC-based approaches).
While these approaches do mitigate bias accumulation,
they now suffer from \emph{higher variance} as the Monte Carlo horizon grows.
Sometimes, it is possible to strike a balance in the bias-variance tradeoff
using $n$-step returns with an appropriately tuned $n$~\citep{rl_sutton2005, sharsa_park2025}.
However, this does not \emph{fundamentally} solve the curse of horizon,
as it reduces the horizon length only by a constant factor ($n$).
We also need to carefully tune $n$ for each task,
and it requires optimality in length-$n$ trajectory chunks.

Is there an ``ideal'' off-policy RL method that scales to arbitrarily long-horizon tasks,
being free from the curse of horizon?
Our main hypothesis in this work is that a \textbf{divide-and-conquer} algorithm
might provide a path to an answer.
By divide-and-conquer, we mean an algorithm that incrementally learns values by combining two equal-sized trajectory segments into a larger one.
In principle, such an algorithm might be able to \emph{logarithmically} reduce the number of biased recursions, unlike TD-based methods,
while not suffering from high variance, unlike MC-based methods (\Cref{fig:teaser}).

As an initial step to verify this hypothesis,
we focus on an important subset of RL problems, offline goal-conditioned RL (GCRL)~\citep{gcrl_kaelbling1993, ogbench_park2025},
which provides a structure that is naturally amenable to divide and conquer.
Goal-conditioned RL aims to find a policy that can reach any state from any other state in the smallest number of steps.
By its shortest-path-like nature, it satisfies a triangle inequality,
which we can exploit to develop a divide-and-conquer algorithm.

In this work, we propose a new divide-and-conquer-based offline goal-conditioned RL algorithm,
\textbf{Transitive RL (\methodname)}, for long-horizon goal-reaching tasks.
There are several scaling challenges with using the triangle inequality for offline value updates.
We address these challenges with newly introduced ideas,
such as in-trajectory subgoals and distance re-weighting.
Empirically, we show that \methodname often exhibits better performance than previous TD- and MC-based approaches on
highly complex, long-horizon robotics tasks over $1000$ environment steps.
To the best of our knowledge, \methodname is the first divide-and-conquer value learning method
that scales to long-horizon robotic tasks beyond toy domains.

\cutsectionup
\section{Related Work}
\cutsectiondown
\label{sec:related}

\textbf{Offline RL and offline goal-conditioned RL.}
Much work has studied the offline reinforcement learning problem,
which aims to find the best return-maximizing policy within the support of the offline data.
Most works distinguish themselves from the way they mitigate value overestimation for out-of-distribution actions,
usually either imposing conservatism constraints on the value function~\citep{cql_kumar2020},
behavior regularizing the actor~\citep{brac_wu2019, td3bc_fujimoto2021, rebrac_tarasov2023, bdpo_gao2025, fql_park2025},
or implicitly performing Bellman updates without querying out-of-distribution actions~\citep{iql_kostrikov2022, sql_xu2023, xql_garg2023}.
In offline goal-conditioned RL, whose aim is to train a multi-task policy that can reach any state,
previous methods either extend offline RL algorithms to the goal-conditioned setting~\citep{wgcsl_yang2022, hiql_park2023}
or leverage specific structures of the goal-conditioned RL problem
with quasimetric learning~\citep{qrl_wang2023},
hierarchical policy extraction~\citep{hiql_park2023},
or probabilistic interpretation~\citep{cl_eysenbach2021, crl_eysenbach2022, tdinfonce_zheng2024}.
In this work, we propose a new goal-conditioned RL algorithm based on a triangle inequality,
as discussed in the next paragraph.

\textbf{The triangle inequality in GCRL.}
Our main idea in this work is based on the triangle inequality in goal-conditioned RL (\Cref{eq:tri_d}).
The idea of using the triangle inequality
dates back to the very first work in goal-conditioned RL by \citet{gcrl_kaelbling1993}.
Since then, several prior works have leveraged this structure in three main ways.
``Hard'' approaches strictly impose the triangle inequality via an architectural quasimetric constraint~\citep{qrl_wang2023, cmd_myers2024},
``soft'' approaches use the triangle inequality as a value backup~\citep{gcrl_kaelbling1993, fwrl_dhiman2018, sgt_jurgenson2020, coe_piekos2023},
and planning-based approaches employ this inequality for planning~\citep{sorb_eysenbach2019, dcmcts_parascandolo2020, sgt_jurgenson2020, sir_li2021}.
Among planning-free off-policy RL methods,
while hard approaches have proven effective in practice so far (\eg, QRL~\citep{qrl_wang2023}),
soft approaches (to our knowledge) have not yet scaled beyond $2$-D toy environments in the absence of additional planning~\citep{gcrl_kaelbling1993, fwrl_dhiman2018, sgt_jurgenson2020, coe_piekos2023}.

Our method, \methodname, is based on the soft triangle inequality; \ie, we use the triangle inequality for value backups, instead of hard quasimetric constraints.
However, unlike previous soft approaches, we show that \methodname scales to
highly complex, long-horizon robotic tasks with $1$B-sized datasets,
even outperforming hard (\ie, quasimetric-based) approaches.
This is made possible by our key techniques (in-sample maximization within in-sample subgoals; see \Cref{sec:idea}),
which prevent value overestimation when applying the triangle inequality in practice. 

\cutsectionup
\section{Preliminaries}
\cutsectiondown
\label{sec:prelim}

\textbf{Problem setting.}
We consider a controlled Markov process defined by a tuple $\gM = (\gS, \gA, p)$,
where $\gS$ is the state space,
$\gA$ is the action space,
and $p(\pl{s'} \mid \pl{s}, \pl{a}): \gS \times \gA \to \Delta(\gS)$ is the transition dynamics distribution.
In the definition above, $\Delta(\gX)$ denotes the set of probability distributions on a space $\gX$
and placeholder variables are denoted in gray.
We assume that we are given an unlabeled dataset $\gD = \{\tau^{(i)}\}_{i}$
consisting of length-$T$ reward-free trajectories $\tau = (s_0, a_0, s_1, \ldots, s_T)$
(in practice, different trajectories can have different lengths, but we assume they have the same length for the simplicity of discussion).

In this work, we aim to develop a better recipe for offline goal-conditioned RL.
Offline GCRL aims to learn a goal-conditioned policy $\pi(\pl{a} \mid \pl{s}, \pl{g}): \gS \times \gS \to \gA$
that maximizes the objective
$\tilde V^\pi(s, g)=\E_{\tau \sim p^\pi(\pl{\tau} \mid s, g)}[\sum_{t=0}^T \gamma^t \sI(s_t = g)]$ for all $s, g \in \gS$
purely from the dataset $\gD$ (without additional environment interactions),
where $\mathbb{I}(\cdot)$ denotes the $0$-$1$ indicator function,\footnote{We assume that the state space is discrete for notational simplicity. This is only to simplify mathematical notation, and most discussions in this paper can be directly extended to continuous state spaces as well with appropriate measure-theoretic formulations~\citep{sm_blier2021}.}
$\gamma \in (0, 1)$ denotes the discount factor,
and $p^\pi(\tau \mid s, g) = \pi(a_0 \mid s_0, g)p(s_1 \mid s_0, a_0) \cdots p(s_T \mid s_{T-1}, a_{T-1})$ where $s_0 = s$.
Here, we note that states and goals live in the same space $\gS$,
and the reward function is simply given as the $0$-$1$ sparse reward function.
We also consider a variant of the above goal-conditioned RL objective defined by hitting times,
where the agent enters an absorbing state upon reaching the goal (\ie, gets a reward of $1$ at most once)~\citep{qrl_wang2023}.
We denote the corresponding hitting-time-based value function as $V^\pi(\pl{s}, \pl{g})$.
In this work, we mainly consider this hitting-time variant, as it has a close connection to temporal distances, as discussed in the next paragraph.

As commonly done by many previous works in goal-conditioned RL~\citep{sorb_eysenbach2019, diffuser_janner2022, icvf_ghosh2023, hiql_park2023, qrl_wang2023, hilp_park2024},
we assume that the environment dynamics are deterministic, unless mentioned otherwise.
In deterministic environments, we can define the notion of temporal distances.
The temporal distance $d^*(s, g)$ from $s \in \gS$ to $g \in \gS$ 
is defined as the minimum number of steps to reach $g$ from $s$ ($\infty$ if it is not reachable).
From the definition, it is straightforward to see that $V^*(s, g) := \max_\pi V^\pi (s, g) = \gamma^{d^*(s, g)}$ in the deterministic case.

\subsection{Previous Approaches to Offline GCRL}

An offline RL algorithm typically consists of two components: value estimation and policy extraction.
In this section, we review two representative paradigms for goal-conditioned value learning (MC and TD),
and two standard techniques for policy extraction.

\textbf{Monte Carlo (MC) value estimation.}
Monte Carlo value estimation is one of the simplest techniques for goal-conditioned value learning.
It trains a behavioral value function based on the distances achieved in dataset trajectories.
Among MC-based approaches, distance regression methods~\citep{mbold_tian2021, recon_shah2021}
fit a Q function $Q(\pl{s}, \pl{a}, \pl{g}): \gS \times \gA \times \gS \to \sR$
with the following loss:
\begin{align}
    L^\mathrm{MC}(Q) = \E_{\substack{\tau \sim \gD, \\ 0 \leq i \leq j < T}}\left[(Q(s_i, a_i, s_j) - \gamma^{j-i})^2\right],
    \label{eq:mc}
\end{align}
where $\tau$ is sampled uniformly from the dataset, and $i$ and $j$ are sampled uniformly from the set $\{0, 1, \ldots, T-1\}$.
While this objective is biased when either the dynamics or the data-collecting policy is stochastic~\citep{dist_akella2023},
it is simple, stable, and often performs surprisingly well in practice~\citep{mbold_tian2021, recon_shah2021, dwsl_hejna2023}.

\textbf{Temporal difference (TD) value learning.}
Another class of methods employs temporal difference learning to learn a goal-conditioned value function.
As an example, goal-conditioned implicit Q-learning (GCIQL)~\citep{iql_kostrikov2022, hiql_park2023} fits
value functions with the following losses:
\begin{align}
    L^\mathrm{IQL-V}(V) &= \E_{s, a, g \sim \gD}
    \left[ \ell_\kappa^2(V(s, g) - \bar Q(s, a, g)) \right], \label{eq:iql_v} \\
    L^\mathrm{IQL-Q}(Q) &= \E_{s, a, s', g \sim \gD}
    \left[(Q(s, a, g) - \mathbb{I}(s=g) - \gamma V(s', g))^2\right], \label{eq:iql_q}
\end{align}
where goals are typically sampled with hindsight relabeling~\citep{her_andrychowicz2017, ogbench_park2025},
$\ell_\kappa^2$ denotes the expectile loss with a parameter $\kappa \in [0.5, 1)$, $\ell_\kappa^2(x) = |\kappa - \mathbb{I}(x < 0)|x^2$,
and $\bar Q$ denotes the target Q function~\citep{dqn_mnih2013}.
The asymmetric expectile regression in \Cref{eq:iql_v} approximates the $\max$ operator in the Bellman update
($V(s, g) \gets \max_{a \in \gA} Q(s, a, g)$) only with in-sample actions, avoiding querying the Q function with out-of-distribution actions.
We note that when $\kappa = 0.5$, the above objective yields
a behavioral (SARSA) goal-conditioned value function~\citep{onestep_brandfonbrener2021, iql_kostrikov2022, sharsa_park2025}.

\textbf{Policy extraction.}
After learning a goal-conditioned value function,
the next step is to train a policy to maximize the learned values.
This procedure is called policy extraction, and we describe two standard techniques below.

Reparameterized gradients~\citep{td3bc_fujimoto2021} extract a Gaussian policy by maximizing the following ``DDPG+BC'' objective~\citep{bottleneck_park2024}:
\begin{align}
J^\mathrm{DDPG+BC}(\pi) = \E_{\substack{s, a, g \sim \gD, \\ a^\pi \sim \pi(\pl{a} \mid s, g)}}[Q(s, a^\pi, g) + \alpha \log \pi(a \mid s, g)], \label{eq:ddpgbc}
\end{align}
where $a^\pi$ is a reparameterized sample from the policy and $\alpha$ is the strength of the BC constraint.
Intuitively, this maximizes the learned value function while not deviating too much from the dataset distribution to prevent out-of-distribution exploitation.

Rejection sampling~\citep{emaq_ghasemipour2021, sfbc_chen2023, idql_hansenestruch2023}
simply defines a policy to be the $\argmax$ of behavioral action samples that maximize the value function:
\begin{align}
    \pi(s, g) \stackrel{d}{=} \argmax_{a_1, \cdots, a_N: a_i \sim \pi^\beta(\pl{a} \mid s, g)} Q(s, a_i, g), \label{eq:rejection}
\end{align}
where $N$ denotes the number of samples, $\stackrel{d}{=}$ denotes equality in distribution,
and $\pi^\beta$ denotes a goal-conditioned BC policy.
Typically, the BC policy is modeled by an expressive generative model (\eg, diffusion models~\citep{diffusion_sohl2015, ddpm_ho2020} and flow matching~\citep{flow_lipman2024})
so that it can effectively capture the potentially multi-modal distributions of the behavioral policy~\citep{sfbc_chen2023, idql_hansenestruch2023, sharsa_park2025}.

\cutsectionup
\section{Transitive RL}
\cutsectiondown

As motivated in \Cref{sec:intro},
our main aim is
to develop a practical \emph{divide-and-conquer} value learning algorithm for offline goal-conditioned RL.
Offline goal-conditioned RL provides a natural shortest-path structure that is amenable to divide and conquer.
Namely, in deterministic environments,%
\footnote{While we mainly consider deterministic environments in this work, we note that the GCRL problem exhibits a similar triangle-inequality structure in stochastic environments as well; see \citet{cmd_myers2024}.}
the temporal distance function $d^*(\pl{s}, \pl{g})$
always satisfies the following triangle inequality~\citep{gcrl_kaelbling1993, fwrl_dhiman2018, qrl_wang2023, coe_piekos2023}:
\begin{align}
    d^*(s, g) \leq d^*(s, w) + d^*(w, g) \label{eq:tri_d}
\end{align}
for all $s, w, g \in \gS$.
The equality holds when $w$ is on a shortest path from $s$ to $g$.
Throughout this paper, we call $w$ a \emph{subgoal}.

Equivalently, from the relation $V^*(s, g) = \gamma^{d^*(s, g)}$,
we can rewrite this in terms of the optimal value function $V^*(\pl{s}, \pl{g})$ as follows:
\begin{align}
    V^*(s, g) \geq V^*(s, w) V^*(w, g). \label{eq:tri_v}
\end{align}
This motivates the following \emph{transitive} Bellman update rule for a value function $V(\pl{s}, \pl{g})$:
\begin{align}
    V(s, g) \gets
    \begin{dcases*}
    \gamma^0 & if $s = g$, \\
    \gamma^1 & if $(s, g) \in \gE$, \\
    \max_{w \in \gS} V(s, w)V(w, g) & otherwise,
    \end{dcases*} \label{eq:trl_v}
\end{align}
where $\gE$ (the edge set) denotes
the set of $(s, s') \in \gS \times \gS$ pairs
such that $s \neq s'$ and $s'$ is reachable from $s$ by a single action.
If we initialize the value function with $V(s, g) \leq 0$ for all $s, g \in \gS$,
a standard result in computer science (\eg, the Floyd-Warshall algorithm) implies that
repeatedly applying this operator converges to the optimal value function $V^*(\pl{s}, \pl{g})$~\citep{fwrl_dhiman2018, coe_piekos2023}.

We also have a variant of the above update rule for an action-value function $Q(\pl{s}, \pl{a}, \pl{g})$:
\begin{align}
    Q(s, a, g) \gets
    \begin{dcases*}
    \gamma^0 & if $s = g$, \\
    \gamma^1 & if $g = p(s, a)$ and $s \neq g$, \\
    \max_{w \in \gS, a' \in \gA} Q(s, a, w)Q(w, a', g) & otherwise,
    \end{dcases*} \label{eq:trl_q}
\end{align}
where we slightly abuse the notation to denote $p$ as the deterministic dynamics function.

\textbf{Why do we want to use the triangle inequality?}
As described in \Cref{sec:intro},
the main benefit of the triangle inequality is that we can perform divide and conquer
instead of TD or MC value learning.
Unlike TD learning, which requires $O(T)$ Bellman recursions,
divide and conquer can reduce the number of Bellman recursions to $O(\log T)$ in the best case
(in theory, with optimally chosen subgoals).
Unlike MC learning, which suffers from high variance with long horizons,
divide and conquer performs dynamic programming (\ie, aggregates intermediate values)
and thus mitigates the variance issue.
Hence, we expect that the use of the triangle inequality
potentially leads to a scalable value learning algorithm that scales better in \emph{long-horizon} tasks.

\subsection{The Challenge}

While the transitive Bellman update rule (\Cref{eq:trl_v}) provides
a natural way to perform divide and conquer for goal-conditioned value learning,
directly implementing this idea in practice is not straightforward.
The main problem is the $\max_{w \in \gS}$ operator in \Cref{eq:trl_v}.
While we can simply iterate over the set of states to compute the maximum in the tabular case,
doing so would lead to (potentially catastrophic) value overestimation under function approximation, because the $\max$ operator over a large number of subgoals can easily exploit any single subgoal's positively biased value estimation error.
This issue is further exacerbated in the offline setting,
where value overestimation poses a particularly severe challenge~\citep{offline_levine2020}.

To address this challenge, previous works have mainly considered two practical solutions.
\citet{sgt_jurgenson2020} use a previous iteration of the value network (similar to target networks~\citep{dqn_mnih2013})
for the target value in \Cref{eq:trl_v}.
\citet{coe_piekos2023} train a separate generator network $G(\pl{s}, \pl{a}, \pl{g}): \gS \times \gA \times \gS \to \gS$
to approximate the $\max_{w \in \gS}$ operator, similar to DDPG~\citep{ddpg_lillicrap2016}.
However, as we will show in our experiments,
these techniques are not sufficient to stabilize training,
leading to zero success rates in most simulated robotic benchmark tasks. 
Indeed, most prior works that use the transitive Bellman update without explicit planning
have demonstrated their methods only on $2$-D toy tasks, such as point mazes and discrete grid worlds~\citep{gcrl_kaelbling1993, fwrl_dhiman2018, sgt_jurgenson2020, coe_piekos2023}.

\subsection{The Idea}
\label{sec:idea}

Our key idea to handle the overestimation issue in the $\max$ operator
is to perform \emph{in-sample maximization} using only \emph{in-trajectory subgoals}.
Specifically, we apply the following two modifications.%
\footnote{In this section, for notational simplicity, we describe our high-level idea based on the V version of the transitive Bellman operator (\Cref{eq:trl_v}).
In practice, we use the Q version (\Cref{eq:trl_q}).}

First, we replace the strict $\max$ operator in \Cref{eq:trl_v} with soft expectile regression~\citep{exp_newey1987, iql_kostrikov2022}.
That is, we (conceptually) minimize the following loss:
\begin{align}
    \E[\ell^2_\kappa (V(s, g) - \bar V(s, w) \bar V(w, g))], \label{eq:trl_v_exp}
\end{align}
where $\ell^2_\kappa$ denotes the expectile loss described in \Cref{sec:prelim},
and $\bar V$ denotes the target network~\citep{dqn_mnih2013}.
This expectile loss allows us to approximate the maximum in \Cref{eq:trl_v}
without having to iterate over the states explicitly.

Second, we only consider \emph{in-trajectory} states as potential subgoals (we call them \emph{behavioral subgoals}).
That is, for a dataset trajectory $\tau = (s_0, a_0, s_1, \ldots, s_T)$,
we update $V(s_i, s_j)$ only using $V(s_i, s_k)V(s_k, s_j)$ with $i \leq k \leq j$,
instead of considering arbitrary states as subgoals.
While this decision might seem restrictive at first glance,
we find that restricting to behavioral subgoals is crucial
to make the transitive Bellman update work in practice (\Cref{sec:exp_abl}).
Otherwise, the probability of an arbitrary state being a ``valid'' subgoal is low,
and thus we would need much more aggressive maximization (\ie, a higher expectile in \Cref{eq:trl_v_exp}),
which leads to unstable training.
Moreover, we experimentally find that
behavioral subgoals can still be highly effective
even when the dataset consists of uniformly random atomic motions.
This is analogous to how behavioral value functions (\ie, one-step RL)
are often sufficient in practice~\citep{onestep_brandfonbrener2021, crl_eysenbach2022, sharsa_park2025}
even on suboptimal datasets.

\subsection{Practical Recipe}

Based on the high-level idea in \Cref{sec:idea},
we now describe the full recipe of our method,
which we call \textbf{Transitive RL (\methodname)}.
\methodname trains a goal-conditioned action-value function $Q(\pl{s}, \pl{a}, \pl{g}): \gS \times \gA \times \gS \to \sR$
based on the Q-based transitive Bellman operator (\Cref{eq:trl_q}),
and then extracts a goal-conditioned policy $\pi(\pl{a} \mid \pl{s}, \pl{g}): \gS \times \gS \to \Delta(\gA)$.

\subsubsection{Value Learning}

\methodname's value learning is based on the high-level idea described in \Cref{sec:idea}.
In practice, we additionally employ the following technique to further improve performance and stability.

\textbf{Distance-based re-weighting.}
In transitive Bellman updates,
the accuracy of the target value for a longer trajectory chunk ($s_i$ to $s_j$)
depends on those of two shorter trajectory chunks ($s_i$ to $s_k$ and $s_k$ to $s_j$).
Hence, it is particularly important to keep the shorter trajectory chunks' values accurate.
To this end, we weigh the loss for each sample $(s_i, s_j)$ in the batch
by $w(s_i, s_j) := 1/(1 + \log_\gamma Q(s_i, a_i, s_j))^\lambda$,
where $\lambda$ is a hyperparameter.
This adjusts the weight for each chunk to be (roughly) inversely proportional to its estimated distance, thus focusing more on shorter trajectory chunks.
This bears similarity to classical dynamic programming where smaller subproblems are solved before larger ones.
With this technique, the final value loss for \methodname is defined as follows:
\begin{align}
L^\mathrm{TRL}(Q) = 
\E_{\tau \sim \gD}\left[w(s_i, s_j) D_\kappa\left(Q(s_i, a_i, s_j), \bar Q(s_i, a_i, s_k) \bar Q(s_k, a_k, s_j)\right)\right].
\label{eq:trl_loss}
\end{align}
where $D$ denotes a loss function (\eg, squared regression, binary cross-entropy, etc.),
and $D_\kappa$ denotes its expectile variant, $D_\kappa(x, y) := |\kappa - \mathbb{I}(x > y)| D(x, y)$.
In our experiments, we use the binary cross-entropy (BCE) loss,
following prior work in goal-conditioned RL~\citep{qtopt_kalashnikov2018, crl_eysenbach2022, sharsa_park2025}.
Also, to handle the base cases,
we replace $\bar Q(s_i, a_i, s_k)$ with $\gamma^{k-i}$ if $k-i \leq 1$
and $\bar Q(s_k, a_k, s_j)$ with $\gamma^{j-k}$ if $j-k \leq 1$
in the above loss.

\subsubsection{Policy Extraction}
After training a goal-conditioned value function,
we extract a policy to maximize the learned value.
In \methodname, we consider two policy extraction methods,
reparameterized gradients and rejection sampling (\Cref{sec:prelim}).
By default, we use reparameterized gradients,
but we find rejection sampling to work better on long-horizon puzzle tasks,
where the behavioral policy is highly multi-modal.

We provide pseudocode for \methodname in \Cref{alg:trl}.

\begin{tcolorbox}[
  enhanced,
  breakable,
  float,
  floatplacement=h,
  title=\textbf{Digression:} Does random midpoint sampling still lead to logarithmic Bellman recursions?,
  colframe=myblue,
  colback=myblue!8,
  coltitle=white,
  parbox=false,
  left=5pt,
  right=5pt,
  toprule=2pt,
  titlerule=1pt,
  leftrule=1pt,
  rightrule=1pt,
  bottomrule=1pt,
]
While we motivated divide-and-conquer value learning from its logarithmic dependency on the horizon in \Cref{sec:intro},
\methodname does not always divide a trajectory into two equal-sized chunks.
It instead \emph{samples} a random subgoal $s_k$ between $s_i$ and $s_j$.
One might wonder whether \methodname's random ``sampling'' variant still gives us $O(\log T)$ recursions.
The answer is still yes (in the ideal, tabular case), and we provide a proof in \Cref{sec:proofs}.
\end{tcolorbox}

\begin{algorithm}[t!]
\caption{Transitive Reinforcement Learning (\methodname)}
\label{alg:trl}
\begin{algorithmic}
\footnotesize

\State
\LComment{\color{myblue} Value learning}
\State Initialize value function $Q(\pl{s}, \pl{a}, \pl{g})$, policy $\pi(\pl{a} \mid \pl{s}, \pl{g})$
\While{not converged}
\State Sample $\tau = (s_0, a_0, s_1, \ldots, s_T) \sim \gD$
\State Sample $i, j \sim \mathrm{Unif}(\{0, 1, \ldots, T-1\})$ such that $i < j$
\State Sample $k \sim \mathrm{Unif}(\{i, i+1, \ldots, j-1\})$
\State Train $Q$ by minimizing $L^\mathrm{\methodname}(Q)$ (\Cref{eq:trl_loss})
\EndWhile

\vspace{5pt}

\LComment{\color{myblue} Policy extraction (can be run in parallel with above)}
\State Extract $\pi$ using either reparameterized gradients (\Cref{eq:ddpgbc}) or rejection sampling (\Cref{eq:rejection})

\end{algorithmic}
\end{algorithm}

\section{Experiments}
\label{sec:exp}

In this section,
we empirically answer several research questions about divide-and-conquer value learning and \methodname.
In \Cref{sec:exp_long},
we compare \methodname's divide-and-conquer update rule with previous TD- and MC-based update rules
on large-scale, long-horizon tasks.
In \Cref{sec:exp_standard},
we compare \methodname with previous offline GCRL methods on a standard benchmark suite.
In \Cref{sec:exp_abl},
we provide various ablation studies of \methodname.
We use four random seeds unless otherwise stated, and present $95\%$ confidence intervals in plots and standard deviations in tables.
In tables, we highlight numbers that are at or above $95\%$ of the best performance in blue,
as in \citet{ogbench_park2025}.

\subsection{Is divide and conquer better than TD and MC on long-horizon tasks?}
\label{sec:exp_long}

\begin{wrapfigure}{r}{0.29\textwidth}
    \centering
    \raisebox{0pt}[\dimexpr\height-1.0\baselineskip\relax]{
        \begin{subfigure}[t]{1.0\linewidth}
        \includegraphics[width=\linewidth]{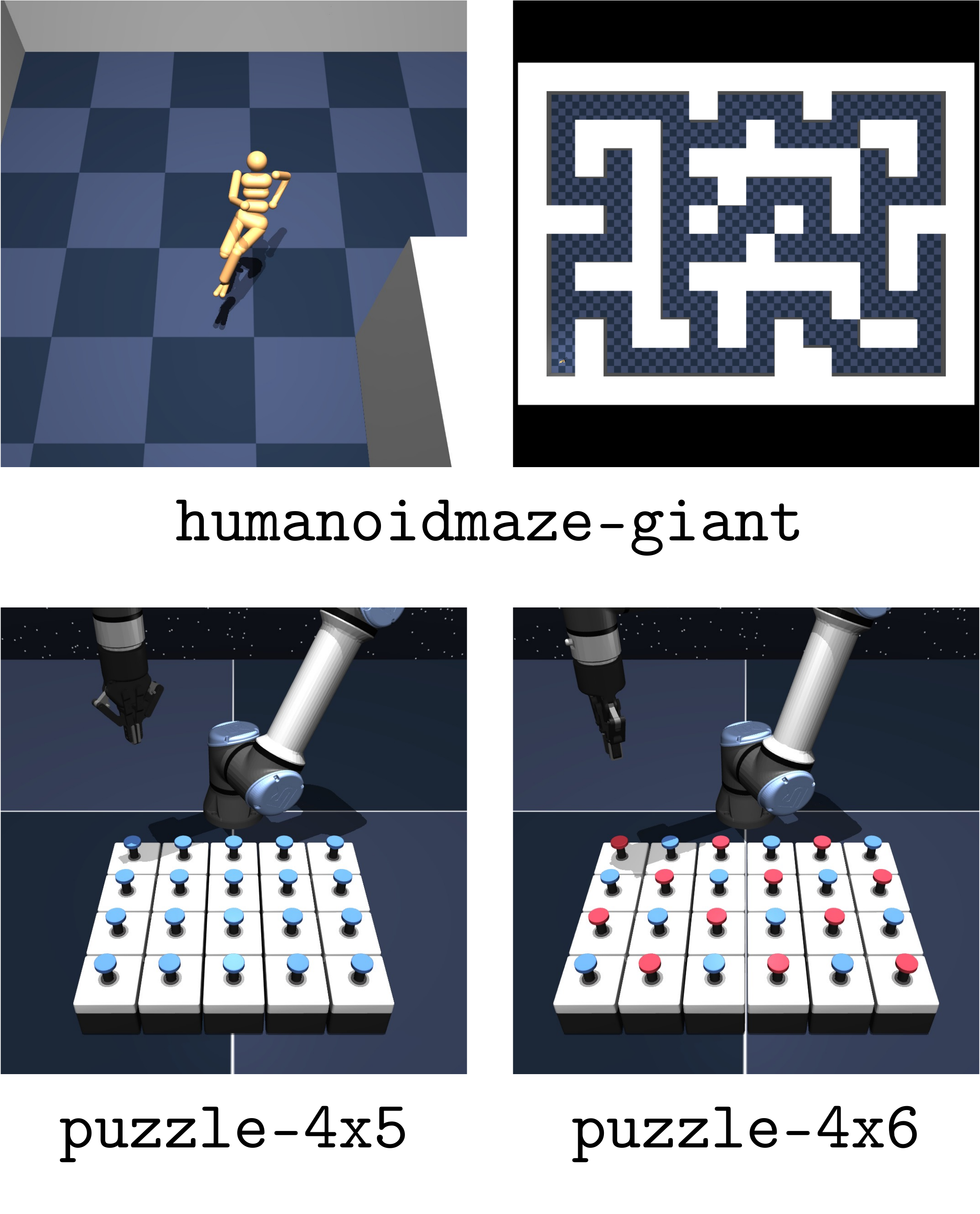}
        \end{subfigure}
    }
    \vspace{-12pt}
    \caption{
    \footnotesize
    \textbf{Long-horizon tasks.}
    }
    \label{fig:1b_envs}
\end{wrapfigure}
Our first goal is to see
whether \methodname's divide-and-conquer value learning algorithm exhibits better scalability to \textbf{long-horizon} problems
than more standard approaches (TD, $n$-step TD, and MC) in practice.
This was our main motivation for \methodname (\Cref{sec:intro}).

\textbf{Tasks.} To empirically answer this question,
we employ three highly complex, long-horizon environments used in a recent horizon scaling study by \citet{sharsa_park2025}:
\tt{humanoidmaze-giant}, \tt{puzzle-4x5}, and \tt{puzzle-4x6} (\Cref{fig:1b_envs}).\footnote{%
We omit \tt{cube-octuple} from \citet{sharsa_park2025}, as none of the methods achieve non-trivial performance on this task without hierarchical \emph{policies}, which are orthogonal to the focus of this paper (\ie, value learning).
}
In \tt{humanoidmaze-giant}, the agent must control a humanoid robot with $21$ degrees of freedom
to navigate a large maze.
In \tt{puzzle-4x5} and \tt{puzzle-4x6}, the agent must control a robot arm
to press buttons to solve the ``Lights Out'' puzzle.
Each environment provides five tasks (\ie, five state-goal pairs) for evaluation.
These tasks are highly challenging and have long horizons:
for example, the test-time task horizon of \tt{humanoidmaze-giant} is $4000$.
For datasets, we employ the $1$B-sized ones used by \citet{sharsa_park2025}.

\textbf{Methods.} In these environments,
we mainly compare the divide-and-conquer value update rule of \methodname with existing value learning paradigms.
\begin{itemize}
\item \textbf{TRL} (ours), which employs divide-and-conquer value learning (\Cref{eq:trl_loss}).
\item \textbf{TD} and \textbf{TD-$\bm{n}$}, which employ ($n$-step) temporal difference value learning (\Cref{sec:methods}).
\item \textbf{MC}, which employs Monte Carlo value learning (\Cref{eq:mc}).
\end{itemize}

To conduct a controlled experiment that solely focuses on the value update rules,
we fix the other hyperparameters, \eg, goal hindsight relabeling ratios and policy extraction methods, except for the BC coefficient $\alpha$ in DDPG+BC (\Cref{eq:ddpgbc}), which we individually tune for each method and task.
This makes the results fully compatible, enabling an apples-to-apples comparison of different value learning paradigms.

Additionally, for reference, we also evaluate several standard offline GCRL algorithms on these long-horizon tasks.
We consider four methods considered in the work by \citet{sharsa_park2025},
flow BC (FBC)~\citep{diffusionpolicy_chi2023},
IQL~\citep{iql_kostrikov2022},
CRL~\citep{crl_eysenbach2022},
and SAC+BC~\citep{sharsa_park2025},
and three existing methods that employ the soft or hard triangle inequality,
QRL~\citep{qrl_wang2023},
TDP~\citep{gcrl_kaelbling1993, fwrl_dhiman2018, sgt_jurgenson2020},
and COE~\citep{coe_piekos2023}.
Among them, IQL and SAC+BC are based on TD learning, and CRL is based on MC learning.

\begin{table}[t!]
\caption{
\footnotesize
\textbf{Results on large-scale, \emph{long-horizon} tasks.}
TRL achieves the best performance on highly challenging, long-horizon benchmark tasks that require up to $3000$ environment steps.
}
\label{table:1b}
\centering
\scalebox{0.79}
{

\begin{tabularew}{ll*{10}{>{\spew{.5}{+1}}r@{\,}l}}
\toprule
\multicolumn{1}{l}{\tt{Environment}} & \multicolumn{1}{l}{\tt{Task}} & \multicolumn{2}{c}{\tt{FBC}} & \multicolumn{2}{c}{\tt{IQL}} & \multicolumn{2}{c}{\tt{CRL}} & \multicolumn{2}{c}{\tt{SAC+BC}} & \multicolumn{2}{c}{\tt{QRL}} & \multicolumn{2}{c}{\tt{TDP}} & \multicolumn{2}{c}{\tt{COE}} & \multicolumn{2}{c}{\tt{TD}} & \multicolumn{2}{c}{\tt{MC}} & \multicolumn{2}{c}{\tt{\color{myblue}TRL}} \\
\midrule
\multirow[c]{6}{*}{\tt{humanoidmaze-giant}} & \tt{task1} & \tt{0} &{\tiny $\pm$\tt{0}} & \tt{0} &{\tiny $\pm$\tt{0}} & \tt{52} &{\tiny $\pm$\tt{38}} & \tt{5} &{\tiny $\pm$\tt{3}} & \tt{1} &{\tiny $\pm$\tt{1}} & \tt{0} &{\tiny $\pm$\tt{0}} & \tt{1} &{\tiny $\pm$\tt{1}} & \tt{2} &{\tiny $\pm$\tt{4}} & \tt{64} &{\tiny $\pm$\tt{9}} & \tt{\color{myblue}{71}} &{\tiny $\pm$\tt{15}} \\
 & \tt{task2} & \tt{0} &{\tiny $\pm$\tt{0}} & \tt{3} &{\tiny $\pm$\tt{7}} & \tt{63} &{\tiny $\pm$\tt{46}} & \tt{12} &{\tiny $\pm$\tt{3}} & \tt{2} &{\tiny $\pm$\tt{2}} & \tt{3} &{\tiny $\pm$\tt{3}} & \tt{2} &{\tiny $\pm$\tt{2}} & \tt{8} &{\tiny $\pm$\tt{4}} & \tt{\color{myblue}{87}} &{\tiny $\pm$\tt{5}} & \tt{\color{myblue}{87}} &{\tiny $\pm$\tt{6}} \\
 & \tt{task3} & \tt{0} &{\tiny $\pm$\tt{0}} & \tt{5} &{\tiny $\pm$\tt{6}} & \tt{68} &{\tiny $\pm$\tt{46}} & \tt{8} &{\tiny $\pm$\tt{6}} & \tt{2} &{\tiny $\pm$\tt{3}} & \tt{3} &{\tiny $\pm$\tt{2}} & \tt{4} &{\tiny $\pm$\tt{2}} & \tt{1} &{\tiny $\pm$\tt{1}} & \tt{\color{myblue}{83}} &{\tiny $\pm$\tt{8}} & \tt{44} &{\tiny $\pm$\tt{8}} \\
 & \tt{task4} & \tt{0} &{\tiny $\pm$\tt{0}} & \tt{2} &{\tiny $\pm$\tt{3}} & \tt{57} &{\tiny $\pm$\tt{41}} & \tt{2} &{\tiny $\pm$\tt{3}} & \tt{0} &{\tiny $\pm$\tt{0}} & \tt{0} &{\tiny $\pm$\tt{0}} & \tt{1} &{\tiny $\pm$\tt{2}} & \tt{2} &{\tiny $\pm$\tt{2}} & \tt{78} &{\tiny $\pm$\tt{8}} & \tt{\color{myblue}{94}} &{\tiny $\pm$\tt{4}} \\
 & \tt{task5} & \tt{0} &{\tiny $\pm$\tt{0}} & \tt{3} &{\tiny $\pm$\tt{4}} & \tt{68} &{\tiny $\pm$\tt{46}} & \tt{0} &{\tiny $\pm$\tt{0}} & \tt{8} &{\tiny $\pm$\tt{8}} & \tt{4} &{\tiny $\pm$\tt{2}} & \tt{5} &{\tiny $\pm$\tt{2}} & \tt{3} &{\tiny $\pm$\tt{1}} & \tt{84} &{\tiny $\pm$\tt{11}} & \tt{\color{myblue}{99}} &{\tiny $\pm$\tt{1}} \\
 & \tt{overall} & \tt{0} &{\tiny $\pm$\tt{0}} & \tt{3} &{\tiny $\pm$\tt{4}} & \tt{62} &{\tiny $\pm$\tt{42}} & \tt{5} &{\tiny $\pm$\tt{0}} & \tt{3} &{\tiny $\pm$\tt{2}} & \tt{2} &{\tiny $\pm$\tt{0}} & \tt{2} &{\tiny $\pm$\tt{1}} & \tt{3} &{\tiny $\pm$\tt{2}} & \tt{\color{myblue}{79}} &{\tiny $\pm$\tt{4}} & \tt{\color{myblue}{79}} &{\tiny $\pm$\tt{2}} \\
\cmidrule{1-22}
\multirow[c]{6}{*}{\tt{puzzle-4x5}} & \tt{task1} & \tt{0} &{\tiny $\pm$\tt{0}} & \tt{\color{myblue}{100}} &{\tiny $\pm$\tt{0}} & \tt{7} &{\tiny $\pm$\tt{5}} & \tt{\color{myblue}{95}} &{\tiny $\pm$\tt{3}} & \tt{0} &{\tiny $\pm$\tt{0}} & \tt{0} &{\tiny $\pm$\tt{0}} & \tt{0} &{\tiny $\pm$\tt{0}} & \tt{84} &{\tiny $\pm$\tt{7}} & \tt{\color{myblue}{100}} &{\tiny $\pm$\tt{0}} & \tt{\color{myblue}{100}} &{\tiny $\pm$\tt{0}} \\
 & \tt{task2} & \tt{0} &{\tiny $\pm$\tt{0}} & \tt{0} &{\tiny $\pm$\tt{0}} & \tt{0} &{\tiny $\pm$\tt{0}} & \tt{0} &{\tiny $\pm$\tt{0}} & \tt{0} &{\tiny $\pm$\tt{0}} & \tt{0} &{\tiny $\pm$\tt{0}} & \tt{0} &{\tiny $\pm$\tt{0}} & \tt{4} &{\tiny $\pm$\tt{3}} & \tt{67} &{\tiny $\pm$\tt{12}} & \tt{\color{myblue}{99}} &{\tiny $\pm$\tt{1}} \\
 & \tt{task3} & \tt{0} &{\tiny $\pm$\tt{0}} & \tt{0} &{\tiny $\pm$\tt{0}} & \tt{0} &{\tiny $\pm$\tt{0}} & \tt{0} &{\tiny $\pm$\tt{0}} & \tt{0} &{\tiny $\pm$\tt{0}} & \tt{0} &{\tiny $\pm$\tt{0}} & \tt{0} &{\tiny $\pm$\tt{0}} & \tt{2} &{\tiny $\pm$\tt{2}} & \tt{8} &{\tiny $\pm$\tt{10}} & \tt{\color{myblue}{100}} &{\tiny $\pm$\tt{0}} \\
 & \tt{task4} & \tt{0} &{\tiny $\pm$\tt{0}} & \tt{0} &{\tiny $\pm$\tt{0}} & \tt{0} &{\tiny $\pm$\tt{0}} & \tt{0} &{\tiny $\pm$\tt{0}} & \tt{0} &{\tiny $\pm$\tt{0}} & \tt{0} &{\tiny $\pm$\tt{0}} & \tt{0} &{\tiny $\pm$\tt{0}} & \tt{2} &{\tiny $\pm$\tt{2}} & \tt{50} &{\tiny $\pm$\tt{9}} & \tt{\color{myblue}{99}} &{\tiny $\pm$\tt{1}} \\
 & \tt{task5} & \tt{0} &{\tiny $\pm$\tt{0}} & \tt{0} &{\tiny $\pm$\tt{0}} & \tt{0} &{\tiny $\pm$\tt{0}} & \tt{0} &{\tiny $\pm$\tt{0}} & \tt{0} &{\tiny $\pm$\tt{0}} & \tt{0} &{\tiny $\pm$\tt{0}} & \tt{0} &{\tiny $\pm$\tt{0}} & \tt{2} &{\tiny $\pm$\tt{2}} & \tt{8} &{\tiny $\pm$\tt{11}} & \tt{\color{myblue}{88}} &{\tiny $\pm$\tt{8}} \\
 & \tt{overall} & \tt{0} &{\tiny $\pm$\tt{0}} & \tt{20} &{\tiny $\pm$\tt{0}} & \tt{1} &{\tiny $\pm$\tt{1}} & \tt{19} &{\tiny $\pm$\tt{1}} & \tt{0} &{\tiny $\pm$\tt{0}} & \tt{0} &{\tiny $\pm$\tt{0}} & \tt{0} &{\tiny $\pm$\tt{0}} & \tt{19} &{\tiny $\pm$\tt{1}} & \tt{47} &{\tiny $\pm$\tt{7}} & \tt{\color{myblue}{97}} &{\tiny $\pm$\tt{1}} \\
\cmidrule{1-22}
\multirow[c]{6}{*}{\tt{puzzle-4x6}} & \tt{task1} & \tt{0} &{\tiny $\pm$\tt{0}} & \tt{87} &{\tiny $\pm$\tt{9}} & \tt{0} &{\tiny $\pm$\tt{0}} & \tt{48} &{\tiny $\pm$\tt{39}} & \tt{0} &{\tiny $\pm$\tt{0}} & \tt{0} &{\tiny $\pm$\tt{0}} & \tt{0} &{\tiny $\pm$\tt{0}} & \tt{61} &{\tiny $\pm$\tt{16}} & \tt{\color{myblue}{98}} &{\tiny $\pm$\tt{4}} & \tt{\color{myblue}{100}} &{\tiny $\pm$\tt{0}} \\
 & \tt{task2} & \tt{0} &{\tiny $\pm$\tt{0}} & \tt{0} &{\tiny $\pm$\tt{0}} & \tt{0} &{\tiny $\pm$\tt{0}} & \tt{5} &{\tiny $\pm$\tt{10}} & \tt{0} &{\tiny $\pm$\tt{0}} & \tt{0} &{\tiny $\pm$\tt{0}} & \tt{0} &{\tiny $\pm$\tt{0}} & \tt{2} &{\tiny $\pm$\tt{2}} & \tt{46} &{\tiny $\pm$\tt{30}} & \tt{\color{myblue}{66}} &{\tiny $\pm$\tt{13}} \\
 & \tt{task3} & \tt{0} &{\tiny $\pm$\tt{0}} & \tt{0} &{\tiny $\pm$\tt{0}} & \tt{0} &{\tiny $\pm$\tt{0}} & \tt{0} &{\tiny $\pm$\tt{0}} & \tt{0} &{\tiny $\pm$\tt{0}} & \tt{0} &{\tiny $\pm$\tt{0}} & \tt{0} &{\tiny $\pm$\tt{0}} & \tt{0} &{\tiny $\pm$\tt{0}} & \tt{34} &{\tiny $\pm$\tt{10}} & \tt{\color{myblue}{67}} &{\tiny $\pm$\tt{21}} \\
 & \tt{task4} & \tt{0} &{\tiny $\pm$\tt{0}} & \tt{0} &{\tiny $\pm$\tt{0}} & \tt{0} &{\tiny $\pm$\tt{0}} & \tt{0} &{\tiny $\pm$\tt{0}} & \tt{0} &{\tiny $\pm$\tt{0}} & \tt{0} &{\tiny $\pm$\tt{0}} & \tt{0} &{\tiny $\pm$\tt{0}} & \tt{1} &{\tiny $\pm$\tt{1}} & \tt{5} &{\tiny $\pm$\tt{6}} & \tt{\color{myblue}{23}} &{\tiny $\pm$\tt{7}} \\
 & \tt{task5} & \tt{\color{myblue}{0}} &{\tiny $\pm$\tt{0}} & \tt{\color{myblue}{0}} &{\tiny $\pm$\tt{0}} & \tt{\color{myblue}{0}} &{\tiny $\pm$\tt{0}} & \tt{\color{myblue}{0}} &{\tiny $\pm$\tt{0}} & \tt{\color{myblue}{0}} &{\tiny $\pm$\tt{0}} & \tt{\color{myblue}{0}} &{\tiny $\pm$\tt{0}} & \tt{\color{myblue}{0}} &{\tiny $\pm$\tt{0}} & \tt{\color{myblue}{0}} &{\tiny $\pm$\tt{0}} & \tt{\color{myblue}{0}} &{\tiny $\pm$\tt{0}} & \tt{\color{myblue}{0}} &{\tiny $\pm$\tt{0}} \\
 & \tt{overall} & \tt{0} &{\tiny $\pm$\tt{0}} & \tt{17} &{\tiny $\pm$\tt{2}} & \tt{0} &{\tiny $\pm$\tt{0}} & \tt{11} &{\tiny $\pm$\tt{8}} & \tt{0} &{\tiny $\pm$\tt{0}} & \tt{0} &{\tiny $\pm$\tt{0}} & \tt{0} &{\tiny $\pm$\tt{0}} & \tt{13} &{\tiny $\pm$\tt{3}} & \tt{37} &{\tiny $\pm$\tt{4}} & \tt{\color{myblue}{51}} &{\tiny $\pm$\tt{5}} \\
\bottomrule
\end{tabularew}

}
\end{table}

\begin{figure}[t!]
    \centering
    \vspace{5pt}
    \includegraphics[width=1.0\textwidth]{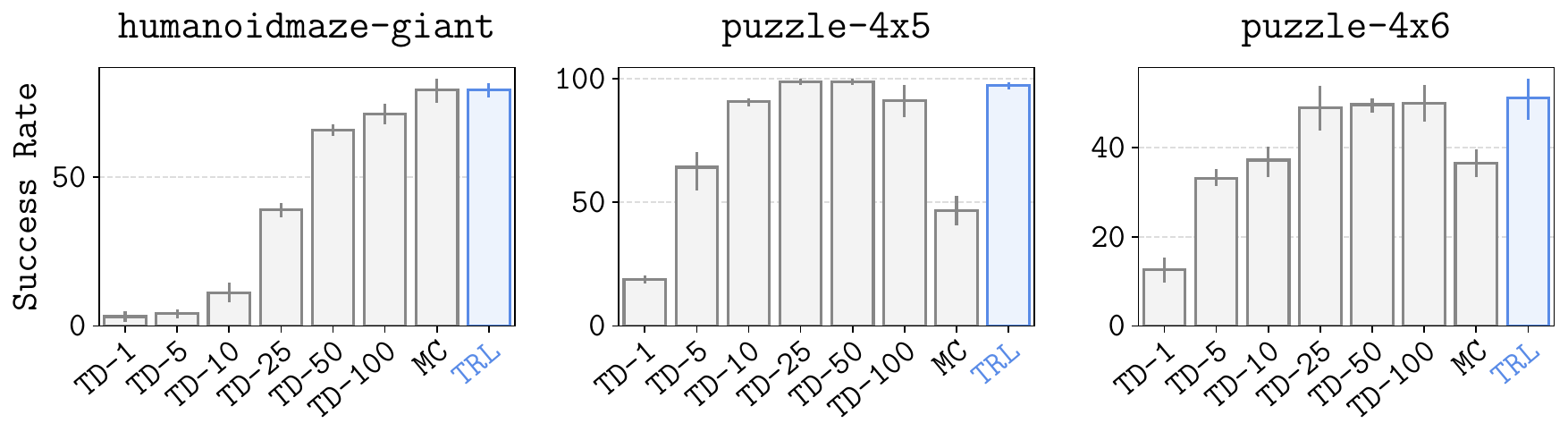}
    \vspace{-10pt}
    \caption{
    \footnotesize
    \textbf{TRL matches the best, individually tuned TD-$\bm{n}$ baseline, without needing to set $\bm{n}$.}
    }
    \label{fig:1b}
\end{figure}

\textbf{Results.}
We present the main comparison results in \Cref{table:1b}
and the comparison results with TD-$\{1, 5, 10, 25, 50, 100\}$ in \Cref{fig:1b}.
The results suggest that \methodname achieves the best performance across all three environments,
outperforming or matching both TD- and MC-based approaches as well as other previous GCRL methods.
Notably, \methodname is the only triangle inequality-based method
that achieves non-trivial performance on these challenging tasks.

In particular, we highlight that $1$-step TD learning struggles on long-horizon tasks.
While TD-$n$ or MC value learning can improve performance,
\Cref{fig:1b} indicates that it requires careful, task-dependent tuning of $n$.
In contrast, \methodname achieves the best performance \textbf{without needing to set $\bm{n}$},
matching the performance of the best TD-$n$ baseline individually tuned for each task.
This highlights the benefits of our divide-and-conquer value learning framework.

\subsection{How does \methodname compare to prior methods on standard benchmarks?}
\label{sec:exp_standard}

Next, we compare the performance of \methodname with existing offline GCRL algorithms on OGBench~\citep{ogbench_park2025},
a standard benchmark in offline goal-conditioned RL.
The main research question here is
whether \methodname is comparable to existing algorithms on \textbf{regular, not necessarily long-horizon} tasks.

\textbf{Tasks.}
We consider $10$ goal-reaching tasks from OGBench.
They span across robotic maze navigation (\tt{\{point, ant, humanoid\}maze}),
ball control (\tt{antsoccer}),
robotic object manipulation (\tt{cube, scene})
and puzzle solving (\tt{puzzle}).
We use the \tt{oraclerep} variants of these tasks, as in \citet{sharsa_park2025}.
For datasets, we use the standard \tt{navigate} and \tt{play} datasets
consisting of task-agnostic trajectories that randomly navigate the maze or perform random tasks. 

\textbf{Methods.}
We consider five widely used offline GCRL methods in this section:
BC,
FBC,
IVL~\citep{iql_kostrikov2022, hiql_park2023},
IQL~\citep{iql_kostrikov2022},
CRL~\citep{crl_eysenbach2022},
and QRL~\citep{qrl_wang2023}.
As in the previous section,
we additionally consider two algorithms that employ the soft triangle inequality,
TDP~\citep{gcrl_kaelbling1993, fwrl_dhiman2018, sgt_jurgenson2020} and COE~\citep{coe_piekos2023}.

\begin{figure}[h!]
    \centering
    \includegraphics[width=1.0\textwidth]{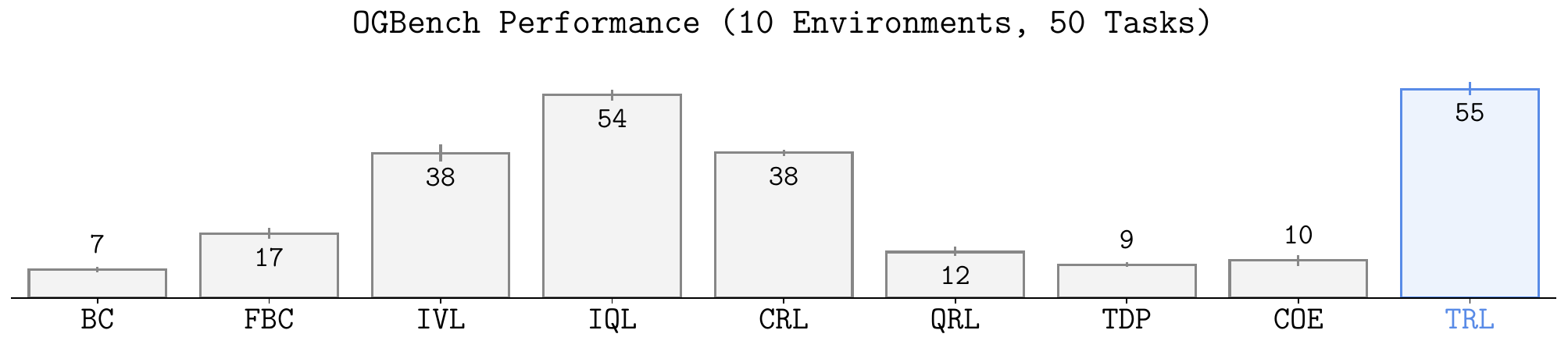}
    \caption{
    \footnotesize
    \textbf{TRL achieves strong performance on standard OGBench tasks.}
    While TRL is not specifically designed for short-horizon tasks,
    it outperforms or matches previous GCRL methods on a standard benchmark.
    }
    \label{fig:standard}
\end{figure}

\textbf{Results.}
\Cref{fig:standard} shows the aggregated performance of each method
over $10$ OGBench environments and $50$ evaluation tasks
(see \Cref{table:standard} for the full table).
The results suggest that \methodname achieves the best performance on average.
While the gap between \methodname and the second-best method (IQL) is narrower than that in \Cref{sec:exp_long},
this is as expected to some extent, since \methodname has its strongest advantages in long-horizon tasks.
We also note that \methodname is the only triangle inequality-based method that achieves strong performance
on these robotic tasks.

\subsection{Ablation Studies}
\label{sec:exp_abl}

We present ablation studies on three components of \methodname
(expectile $\kappa$, subgoal distributions, and distance-based re-weighting factor $\lambda$)
in this section.

\begin{wrapfigure}{r}{0.49\textwidth}
    \centering
    \raisebox{0pt}[\dimexpr\height-1.0\baselineskip\relax]{
        \begin{subfigure}[t]{1.0\linewidth}
        \includegraphics[width=\linewidth]{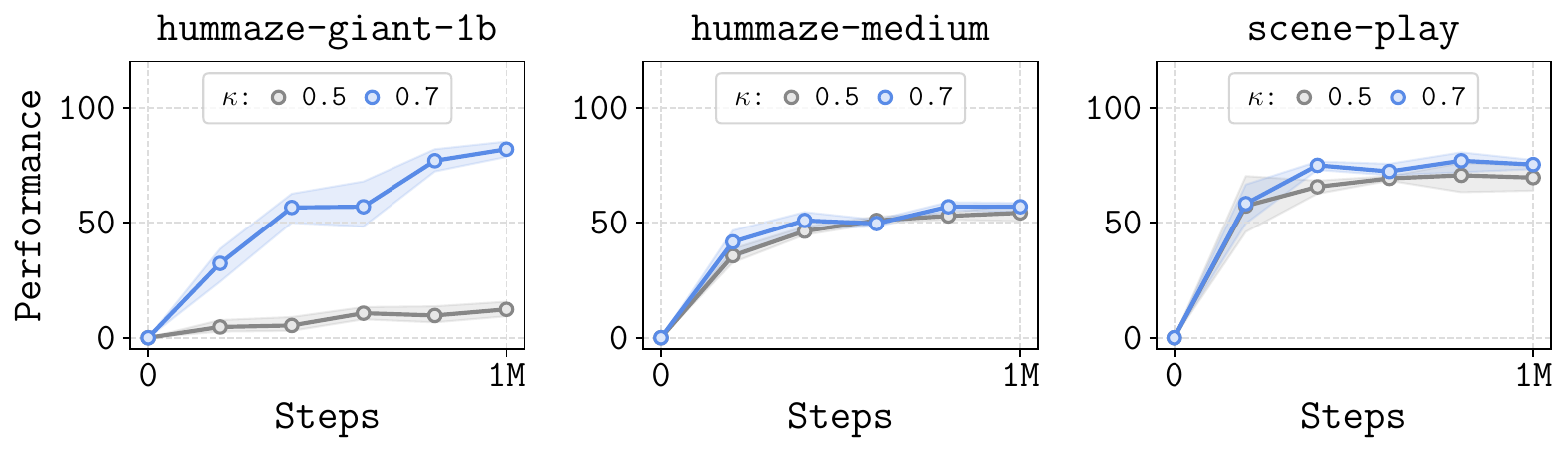}
        \end{subfigure}
    }
    \vspace{-15pt}
    \caption{
    \footnotesize
    \textbf{Ablation study on the expectile $\bm{\kappa}$.}
    }
    \vspace{-5pt}
    \label{fig:abl_expectile}
\end{wrapfigure}
\textbf{Expectile $\bm{\kappa}$.}
\methodname uses expectile regression to approximate the $\max$ operator in \Cref{eq:trl_v}.
We evaluate \methodname with $\kappa \in \{0.5, 0.7\}$, and present the results in \Cref{fig:abl_expectile}.
The results show that
while $\kappa = 0.5$ (\ie, behavioral regression) works well in \tt{humanoidmaze-medium} and \tt{scene-play},
$\kappa > 0.5$ is crucial for achieving strong performance on \tt{humanoidmaze-giant}.
The strong performance with $\kappa = 0.5$ on some tasks
is akin to how one-step RL (\ie, behavioral value learning)
is often enough to achieve solid performance in offline RL~\citep{onestep_brandfonbrener2021, sharsa_park2025}.

\begin{wrapfigure}{r}{0.49\textwidth}
    \centering
    \raisebox{0pt}[\dimexpr\height-1.0\baselineskip\relax]{
        \begin{subfigure}[t]{1.0\linewidth}
        \includegraphics[width=\linewidth]{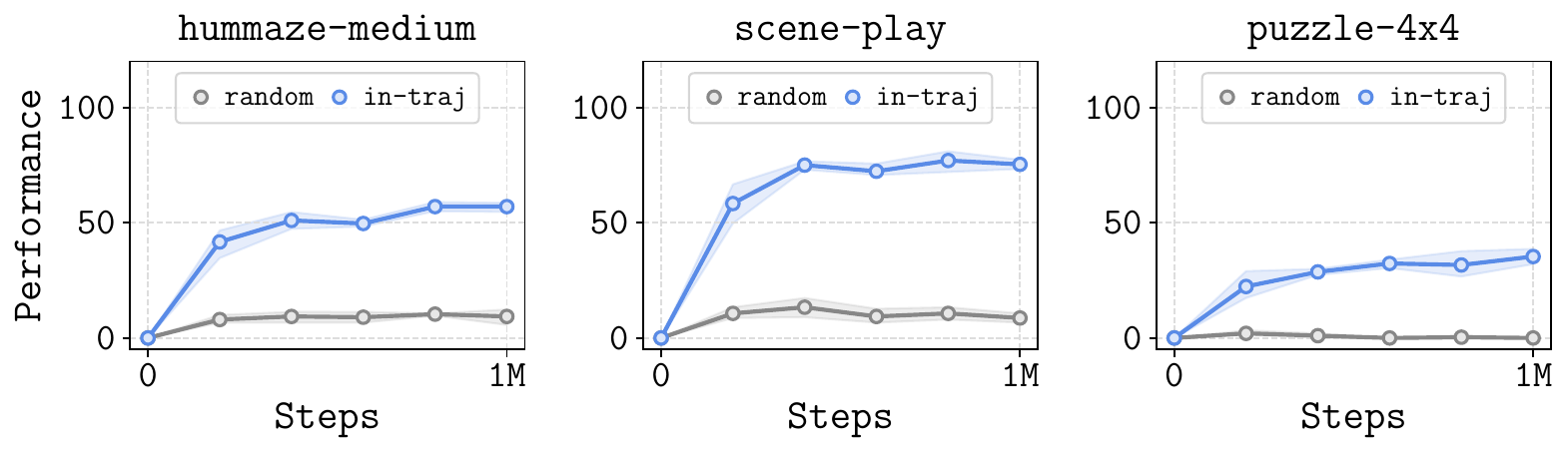}
        \end{subfigure}
    }
    \vspace{-15pt}
    \caption{
    \footnotesize
    \textbf{Ablation study on subgoal distributions.}
    }
    \vspace{-5pt}
    \label{fig:abl_subgoal_type}
\end{wrapfigure}
\textbf{Subgoal distributions.}
Another key feature of \methodname is the use of \emph{behavioral} (in-trajectory) subgoals.
In \Cref{fig:abl_subgoal_type}, we ablate this choice
by comparing behavioral subgoals (``in-traj'') with random subgoals (``random'').
The results suggest that random subgoals substantially degrade performance,
highlighting the importance of using behavioral subgoals.

\clearpage

\begin{wrapfigure}{r}{0.49\textwidth}
    \centering
    \vspace{12pt}
    \raisebox{0pt}[\dimexpr\height-1.0\baselineskip\relax]{
        \begin{subfigure}[t]{1.0\linewidth}
        \includegraphics[width=\linewidth]{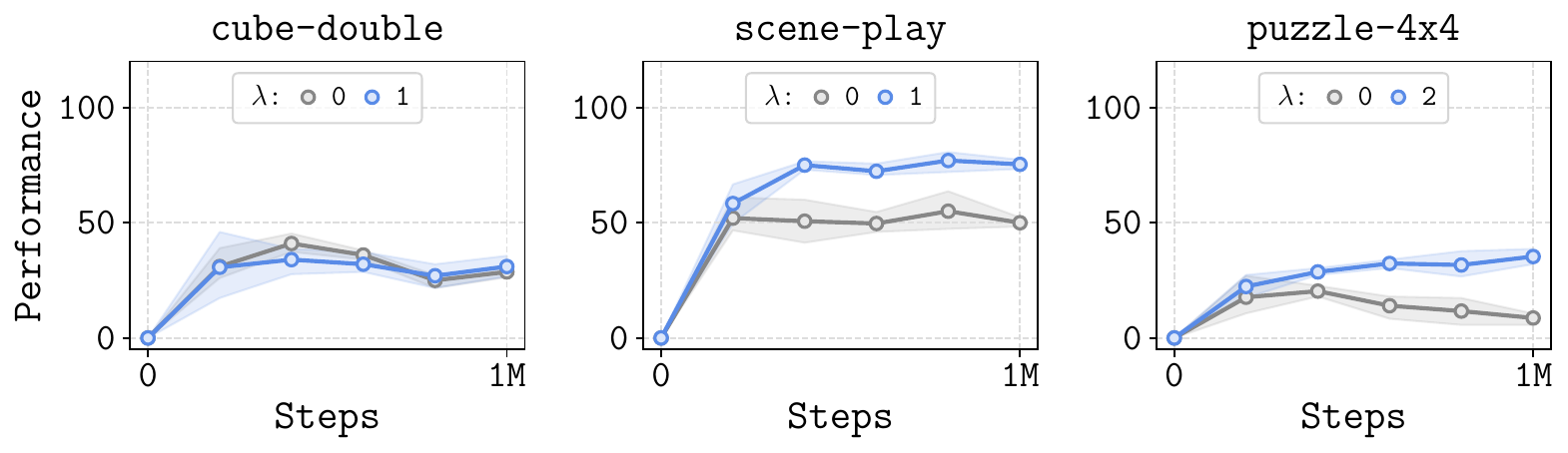}
        \end{subfigure}
    }
    \vspace{-15pt}
    \caption{
    \footnotesize
    \textbf{Ablation study on the weighting factor $\bm{\lambda}$.}
    }
    \vspace{-5pt}
    \label{fig:abl_lam}
\end{wrapfigure}
\textbf{Distance-based re-weighting factor $\bm{\lambda}$.}
\Cref{fig:abl_lam} shows an ablation study on \methodname's distance-based re-weighting factor $\lambda$.
Note that a positive $\lambda$ makes \methodname focus more on shorter trajectory chunks.
The results show that enabling re-weighting (\ie, $\lambda > 0$) often leads to better performance.

\cutsectionup
\section{What's Next?}
\cutsectiondown

In this work,
we have shown that Transitive RL
exhibits better horizon scalability than standard TD- or MC-based value learning approaches
in offline goal-conditioned RL.
At a high level, this provides a piece of affirmative evidence for the hypothesis we posed in \Cref{sec:intro}:
namely, a \emph{divide-and-conquer} value learning algorithm might lead to
an ideal value learning algorithm that is free from the curse of horizon.

This is only the beginning of the journey.
For example, it remains an open question whether a similar divide-and-conquer value learning technique
could be applied to learn an \emph{unbiased} value function in stochastic environments
(which is a limitation of \methodname, as well as many other works leveraging the vanilla triangle inequality in GCRL).
The successor temporal distance framework proposed by \citet{cmd_myers2024} may provide clues to this question.
Another open question is whether \methodname (or any divide-and-conquer-style algorithm) can be extended to general reward-based RL tasks,
beyond goal-conditioned RL. 
We hope the ideas and techniques introduced in this work help address these important open questions
and facilitate future progress toward scalable value learning algorithms.

\section*{Acknowledgment}
This work was partly supported by the Korea Foundation for Advanced Studies (KFAS), the NSF Graduate Research Fellowship, AFOSR FA9550-22-1-0273, and ONR N00014-22-1-2773.
This research used the Savio computational cluster resource provided by the Berkeley Research Computing program at UC Berkeley.

\bibliography{iclr2026_conference}
\bibliographystyle{iclr2026_conference}

\clearpage

\appendix
\section{Proofs}
\label{sec:proofs}

In this section, we show that
\methodname's random sampling variant of divide-and-conquer value learning still gives us a logarithmic dependency on the horizon
in the ideal, tabular case.
Let $B(n)$ be the expected number of \emph{maximum} Bellman recursions to compute the value $V(s_i, s_j)$
for a length-$n$ trajectory chunk, $(s_i, s_{i+1}, \ldots, s_{j = i+n})$
with the random sampling variant of transitive Bellman updates.
Note that $B$ satisfies the following recursive equation:
\begin{align*}
    B(1) &= 0, \\
    B(n) &= 1 + \frac{1}{n-1} \sum_{k=1}^{n-1} \max(B(k), B(n-k)),
\end{align*}
for $n > 1$.

Then, we have the following results:

\begin{lemma}
\label{lem:seq}
Let $n \geq 2$ be an integer and $C(n) = \frac{1}{n-1} \sum_{k=1}^{n-1} \max(k, n-k)$. Then $C(n) \leq 3n/4$.
\end{lemma}
\begin{proof}
If $n = 2m$ (even),
\begin{align*}
    C(n) &= \frac{1}{n-1}\sum_{k=1}^{n-1} \max(k, n-k) \\
    &= \frac{1}{2m-1}(m + 2((m+1) + \cdots + (2m-1))) \\
    &= \frac{1}{2m-1} (m + 3m(m-1)) \\
    &= \frac{3m^2-2m}{2m-1} \\
    &\leq \frac{3}{2}m \\
    &= \frac{3}{4} n.
\end{align*}

If $n = 2m + 1$ (odd),
\begin{align*}
    C(n) &= \frac{1}{n-1}\sum_{k=1}^{n-1} \max(k, n-k) \\
    &= \frac{1}{2m}(2((m+1) + \cdots + 2m)) \\
    &= \frac{1}{2m}(m(3m+1)) \\
    &= \frac{3m+1}{2} \\
    &\leq \frac{6m+3}{4} \\
    &= \frac{3}{4}n.
\end{align*}
\end{proof}

\begin{theorem}
\label{thm:seq}
$B(n) \leq \log n / \log (4/3)$.
\end{theorem}
\begin{proof}
Recall that we have the following recursive equation:
\begin{align*}
    B(1) &= 0, \\
    B(n) &= 1 + \frac{1}{n-1} \sum_{k=1}^{n-1} \max(B(k), B(n-k)),
\end{align*}
for $n > 1$.

First, note that $B(1) = 0 \leq \log 1/\log (4/3)$.
For $n \geq 2$, we proceed by induction, assuming that $B(m) \leq \log m / \log (4/3)$ for all $m < n$.
Then,
\begin{align*}
    B(n) &= 1 + \frac{1}{n-1} \sum_{k=1}^{n-1} \max(B(k), B(n-k)) \\
    &\leq 1 + \frac{1}{n-1} \sum_{k=1}^{n-1} \frac{\max(\log k, \log (n-k))}{\log (4/3)} \\
    &= 1 + \frac{1}{\log (4/3)} \frac{1}{n-1} \sum_{k=1}^{n-1} \log( \max(k, n-k)) \\
    &\leq 1 + \frac{1}{\log (4/3)} \log \left(\frac{1}{n-1} \sum_{k=1}^{n-1} \max(k,n-k)\right) \\
    &= 1 + \frac{1}{\log (4/3)} \log C(n) \\
    &\leq 1 + \frac{1}{\log (4/3)} \log \left(\frac34 n\right) \\
    &= \frac{\log n}{\log (4/3)},
\end{align*}
where we use the inductive hypothesis in the first inequality,
Jensen's inequality in the second,
and \Cref{lem:seq} in the third.
\end{proof}

\Cref{thm:seq} shows that the number of expected maximum Bellman recursions in the random sampling variant of transitive Bellman updates
also has a logarithmic dependency on the horizon length.

\clearpage

\section{Additional Results}

We provide the full comparison results on the standard OGBench tasks (\Cref{sec:exp_standard}) in \Cref{table:standard}.

\begin{table}[h!]
\caption{
\footnotesize
\textbf{Full results on standard OGBench tasks.}
}
\label{table:standard}
\centering
\scalebox{0.64}
{

\begin{tabularew}{ll*{9}{>{\spew{.5}{+1}}r@{\,}l}}
\toprule
\multicolumn{1}{l}{\tt{Environment}} & \multicolumn{1}{l}{\tt{Task}} & \multicolumn{2}{c}{\tt{BC}} & \multicolumn{2}{c}{\tt{FBC}} & \multicolumn{2}{c}{\tt{IVL}} & \multicolumn{2}{c}{\tt{IQL}} & \multicolumn{2}{c}{\tt{CRL}} & \multicolumn{2}{c}{\tt{QRL}} & \multicolumn{2}{c}{\tt{TDP}} & \multicolumn{2}{c}{\tt{COE}} & \multicolumn{2}{c}{\tt{\color{myblue}TRL}} \\
\midrule
\multirow[c]{6}{*}{\tt{pointmaze-large-navigate-oraclerep-v0}} & \tt{task1} & \tt{53} &{\tiny $\pm$\tt{24}} & \tt{98} &{\tiny $\pm$\tt{3}} & \tt{94} &{\tiny $\pm$\tt{6}} & \tt{92} &{\tiny $\pm$\tt{9}} & \tt{23} &{\tiny $\pm$\tt{18}} & \tt{31} &{\tiny $\pm$\tt{28}} & \tt{40} &{\tiny $\pm$\tt{19}} & \tt{67} &{\tiny $\pm$\tt{26}} & \tt{52} &{\tiny $\pm$\tt{19}} \\
 & \tt{task2} & \tt{3} &{\tiny $\pm$\tt{4}} & \tt{17} &{\tiny $\pm$\tt{6}} & \tt{0} &{\tiny $\pm$\tt{0}} & \tt{0} &{\tiny $\pm$\tt{0}} & \tt{42} &{\tiny $\pm$\tt{28}} & \tt{5} &{\tiny $\pm$\tt{6}} & \tt{0} &{\tiny $\pm$\tt{0}} & \tt{0} &{\tiny $\pm$\tt{0}} & \tt{1} &{\tiny $\pm$\tt{2}} \\
 & \tt{task3} & \tt{17} &{\tiny $\pm$\tt{9}} & \tt{76} &{\tiny $\pm$\tt{11}} & \tt{100} &{\tiny $\pm$\tt{0}} & \tt{76} &{\tiny $\pm$\tt{13}} & \tt{67} &{\tiny $\pm$\tt{14}} & \tt{0} &{\tiny $\pm$\tt{0}} & \tt{4} &{\tiny $\pm$\tt{4}} & \tt{12} &{\tiny $\pm$\tt{8}} & \tt{7} &{\tiny $\pm$\tt{5}} \\
 & \tt{task4} & \tt{23} &{\tiny $\pm$\tt{10}} & \tt{79} &{\tiny $\pm$\tt{8}} & \tt{3} &{\tiny $\pm$\tt{5}} & \tt{0} &{\tiny $\pm$\tt{0}} & \tt{13} &{\tiny $\pm$\tt{18}} & \tt{1} &{\tiny $\pm$\tt{2}} & \tt{56} &{\tiny $\pm$\tt{12}} & \tt{44} &{\tiny $\pm$\tt{27}} & \tt{41} &{\tiny $\pm$\tt{13}} \\
 & \tt{task5} & \tt{30} &{\tiny $\pm$\tt{22}} & \tt{87} &{\tiny $\pm$\tt{5}} & \tt{44} &{\tiny $\pm$\tt{39}} & \tt{0} &{\tiny $\pm$\tt{0}} & \tt{21} &{\tiny $\pm$\tt{8}} & \tt{0} &{\tiny $\pm$\tt{0}} & \tt{49} &{\tiny $\pm$\tt{16}} & \tt{48} &{\tiny $\pm$\tt{16}} & \tt{64} &{\tiny $\pm$\tt{14}} \\
 & \tt{overall} & \tt{25} &{\tiny $\pm$\tt{3}} & \tt{71} &{\tiny $\pm$\tt{4}} & \tt{48} &{\tiny $\pm$\tt{10}} & \tt{34} &{\tiny $\pm$\tt{3}} & \tt{33} &{\tiny $\pm$\tt{7}} & \tt{7} &{\tiny $\pm$\tt{7}} & \tt{30} &{\tiny $\pm$\tt{5}} & \tt{34} &{\tiny $\pm$\tt{7}} & \tt{33} &{\tiny $\pm$\tt{5}} \\
\cmidrule{1-20}
\multirow[c]{6}{*}{\tt{antmaze-large-navigate-oraclerep-v0}} & \tt{task1} & \tt{3} &{\tiny $\pm$\tt{3}} & \tt{13} &{\tiny $\pm$\tt{11}} & \tt{11} &{\tiny $\pm$\tt{13}} & \tt{23} &{\tiny $\pm$\tt{13}} & \tt{87} &{\tiny $\pm$\tt{9}} & \tt{51} &{\tiny $\pm$\tt{15}} & \tt{6} &{\tiny $\pm$\tt{5}} & \tt{6} &{\tiny $\pm$\tt{6}} & \tt{57} &{\tiny $\pm$\tt{13}} \\
 & \tt{task2} & \tt{22} &{\tiny $\pm$\tt{12}} & \tt{24} &{\tiny $\pm$\tt{8}} & \tt{19} &{\tiny $\pm$\tt{5}} & \tt{44} &{\tiny $\pm$\tt{20}} & \tt{64} &{\tiny $\pm$\tt{25}} & \tt{69} &{\tiny $\pm$\tt{12}} & \tt{23} &{\tiny $\pm$\tt{6}} & \tt{26} &{\tiny $\pm$\tt{11}} & \tt{66} &{\tiny $\pm$\tt{4}} \\
 & \tt{task3} & \tt{53} &{\tiny $\pm$\tt{10}} & \tt{48} &{\tiny $\pm$\tt{6}} & \tt{55} &{\tiny $\pm$\tt{6}} & \tt{82} &{\tiny $\pm$\tt{6}} & \tt{91} &{\tiny $\pm$\tt{3}} & \tt{96} &{\tiny $\pm$\tt{3}} & \tt{69} &{\tiny $\pm$\tt{9}} & \tt{67} &{\tiny $\pm$\tt{4}} & \tt{80} &{\tiny $\pm$\tt{4}} \\
 & \tt{task4} & \tt{17} &{\tiny $\pm$\tt{9}} & \tt{7} &{\tiny $\pm$\tt{3}} & \tt{4} &{\tiny $\pm$\tt{3}} & \tt{15} &{\tiny $\pm$\tt{7}} & \tt{92} &{\tiny $\pm$\tt{4}} & \tt{57} &{\tiny $\pm$\tt{16}} & \tt{17} &{\tiny $\pm$\tt{10}} & \tt{14} &{\tiny $\pm$\tt{1}} & \tt{8} &{\tiny $\pm$\tt{8}} \\
 & \tt{task5} & \tt{17} &{\tiny $\pm$\tt{5}} & \tt{16} &{\tiny $\pm$\tt{4}} & \tt{18} &{\tiny $\pm$\tt{12}} & \tt{20} &{\tiny $\pm$\tt{7}} & \tt{90} &{\tiny $\pm$\tt{6}} & \tt{61} &{\tiny $\pm$\tt{12}} & \tt{17} &{\tiny $\pm$\tt{4}} & \tt{18} &{\tiny $\pm$\tt{5}} & \tt{18} &{\tiny $\pm$\tt{11}} \\
 & \tt{overall} & \tt{22} &{\tiny $\pm$\tt{4}} & \tt{22} &{\tiny $\pm$\tt{4}} & \tt{21} &{\tiny $\pm$\tt{6}} & \tt{37} &{\tiny $\pm$\tt{8}} & \tt{85} &{\tiny $\pm$\tt{7}} & \tt{67} &{\tiny $\pm$\tt{8}} & \tt{27} &{\tiny $\pm$\tt{3}} & \tt{26} &{\tiny $\pm$\tt{3}} & \tt{46} &{\tiny $\pm$\tt{5}} \\
\cmidrule{1-20}
\multirow[c]{6}{*}{\tt{humanoidmaze-medium-navigate-oraclerep-v0}} & \tt{task1} & \tt{10} &{\tiny $\pm$\tt{7}} & \tt{7} &{\tiny $\pm$\tt{9}} & \tt{29} &{\tiny $\pm$\tt{4}} & \tt{31} &{\tiny $\pm$\tt{9}} & \tt{91} &{\tiny $\pm$\tt{8}} & \tt{8} &{\tiny $\pm$\tt{12}} & \tt{4} &{\tiny $\pm$\tt{3}} & \tt{6} &{\tiny $\pm$\tt{4}} & \tt{76} &{\tiny $\pm$\tt{6}} \\
 & \tt{task2} & \tt{6} &{\tiny $\pm$\tt{4}} & \tt{8} &{\tiny $\pm$\tt{5}} & \tt{37} &{\tiny $\pm$\tt{12}} & \tt{82} &{\tiny $\pm$\tt{9}} & \tt{96} &{\tiny $\pm$\tt{1}} & \tt{22} &{\tiny $\pm$\tt{29}} & \tt{7} &{\tiny $\pm$\tt{3}} & \tt{9} &{\tiny $\pm$\tt{2}} & \tt{96} &{\tiny $\pm$\tt{2}} \\
 & \tt{task3} & \tt{8} &{\tiny $\pm$\tt{5}} & \tt{15} &{\tiny $\pm$\tt{5}} & \tt{16} &{\tiny $\pm$\tt{7}} & \tt{6} &{\tiny $\pm$\tt{6}} & \tt{72} &{\tiny $\pm$\tt{17}} & \tt{30} &{\tiny $\pm$\tt{21}} & \tt{12} &{\tiny $\pm$\tt{3}} & \tt{13} &{\tiny $\pm$\tt{3}} & \tt{8} &{\tiny $\pm$\tt{8}} \\
 & \tt{task4} & \tt{2} &{\tiny $\pm$\tt{1}} & \tt{3} &{\tiny $\pm$\tt{3}} & \tt{0} &{\tiny $\pm$\tt{0}} & \tt{0} &{\tiny $\pm$\tt{0}} & \tt{8} &{\tiny $\pm$\tt{7}} & \tt{9} &{\tiny $\pm$\tt{9}} & \tt{0} &{\tiny $\pm$\tt{0}} & \tt{1} &{\tiny $\pm$\tt{1}} & \tt{11} &{\tiny $\pm$\tt{3}} \\
 & \tt{task5} & \tt{11} &{\tiny $\pm$\tt{5}} & \tt{12} &{\tiny $\pm$\tt{2}} & \tt{39} &{\tiny $\pm$\tt{16}} & \tt{63} &{\tiny $\pm$\tt{12}} & \tt{93} &{\tiny $\pm$\tt{4}} & \tt{20} &{\tiny $\pm$\tt{19}} & \tt{14} &{\tiny $\pm$\tt{4}} & \tt{14} &{\tiny $\pm$\tt{7}} & \tt{92} &{\tiny $\pm$\tt{3}} \\
 & \tt{overall} & \tt{7} &{\tiny $\pm$\tt{2}} & \tt{9} &{\tiny $\pm$\tt{3}} & \tt{24} &{\tiny $\pm$\tt{4}} & \tt{36} &{\tiny $\pm$\tt{3}} & \tt{72} &{\tiny $\pm$\tt{5}} & \tt{18} &{\tiny $\pm$\tt{17}} & \tt{7} &{\tiny $\pm$\tt{0}} & \tt{9} &{\tiny $\pm$\tt{2}} & \tt{57} &{\tiny $\pm$\tt{1}} \\
\cmidrule{1-20}
\multirow[c]{6}{*}{\tt{humanoidmaze-large-navigate-oraclerep-v0}} & \tt{task1} & \tt{1} &{\tiny $\pm$\tt{1}} & \tt{0} &{\tiny $\pm$\tt{0}} & \tt{9} &{\tiny $\pm$\tt{3}} & \tt{10} &{\tiny $\pm$\tt{1}} & \tt{49} &{\tiny $\pm$\tt{17}} & \tt{3} &{\tiny $\pm$\tt{3}} & \tt{0} &{\tiny $\pm$\tt{0}} & \tt{0} &{\tiny $\pm$\tt{0}} & \tt{9} &{\tiny $\pm$\tt{6}} \\
 & \tt{task2} & \tt{0} &{\tiny $\pm$\tt{0}} & \tt{0} &{\tiny $\pm$\tt{0}} & \tt{0} &{\tiny $\pm$\tt{0}} & \tt{0} &{\tiny $\pm$\tt{0}} & \tt{1} &{\tiny $\pm$\tt{1}} & \tt{0} &{\tiny $\pm$\tt{0}} & \tt{0} &{\tiny $\pm$\tt{0}} & \tt{0} &{\tiny $\pm$\tt{0}} & \tt{0} &{\tiny $\pm$\tt{0}} \\
 & \tt{task3} & \tt{3} &{\tiny $\pm$\tt{1}} & \tt{2} &{\tiny $\pm$\tt{1}} & \tt{3} &{\tiny $\pm$\tt{2}} & \tt{11} &{\tiny $\pm$\tt{6}} & \tt{63} &{\tiny $\pm$\tt{25}} & \tt{9} &{\tiny $\pm$\tt{6}} & \tt{3} &{\tiny $\pm$\tt{4}} & \tt{4} &{\tiny $\pm$\tt{3}} & \tt{29} &{\tiny $\pm$\tt{9}} \\
 & \tt{task4} & \tt{3} &{\tiny $\pm$\tt{3}} & \tt{1} &{\tiny $\pm$\tt{1}} & \tt{1} &{\tiny $\pm$\tt{1}} & \tt{2} &{\tiny $\pm$\tt{1}} & \tt{9} &{\tiny $\pm$\tt{8}} & \tt{0} &{\tiny $\pm$\tt{0}} & \tt{2} &{\tiny $\pm$\tt{2}} & \tt{4} &{\tiny $\pm$\tt{2}} & \tt{4} &{\tiny $\pm$\tt{4}} \\
 & \tt{task5} & \tt{2} &{\tiny $\pm$\tt{2}} & \tt{2} &{\tiny $\pm$\tt{1}} & \tt{2} &{\tiny $\pm$\tt{1}} & \tt{0} &{\tiny $\pm$\tt{0}} & \tt{20} &{\tiny $\pm$\tt{17}} & \tt{1} &{\tiny $\pm$\tt{1}} & \tt{0} &{\tiny $\pm$\tt{0}} & \tt{2} &{\tiny $\pm$\tt{2}} & \tt{0} &{\tiny $\pm$\tt{0}} \\
 & \tt{overall} & \tt{2} &{\tiny $\pm$\tt{1}} & \tt{1} &{\tiny $\pm$\tt{1}} & \tt{3} &{\tiny $\pm$\tt{1}} & \tt{5} &{\tiny $\pm$\tt{1}} & \tt{28} &{\tiny $\pm$\tt{5}} & \tt{3} &{\tiny $\pm$\tt{2}} & \tt{1} &{\tiny $\pm$\tt{1}} & \tt{2} &{\tiny $\pm$\tt{0}} & \tt{8} &{\tiny $\pm$\tt{1}} \\
\cmidrule{1-20}
\multirow[c]{6}{*}{\tt{antsoccer-arena-navigate-oraclerep-v0}} & \tt{task1} & \tt{7} &{\tiny $\pm$\tt{4}} & \tt{30} &{\tiny $\pm$\tt{4}} & \tt{76} &{\tiny $\pm$\tt{5}} & \tt{86} &{\tiny $\pm$\tt{6}} & \tt{52} &{\tiny $\pm$\tt{11}} & \tt{20} &{\tiny $\pm$\tt{5}} & \tt{12} &{\tiny $\pm$\tt{1}} & \tt{6} &{\tiny $\pm$\tt{5}} & \tt{89} &{\tiny $\pm$\tt{3}} \\
 & \tt{task2} & \tt{6} &{\tiny $\pm$\tt{3}} & \tt{24} &{\tiny $\pm$\tt{7}} & \tt{62} &{\tiny $\pm$\tt{5}} & \tt{92} &{\tiny $\pm$\tt{4}} & \tt{36} &{\tiny $\pm$\tt{4}} & \tt{21} &{\tiny $\pm$\tt{3}} & \tt{7} &{\tiny $\pm$\tt{6}} & \tt{11} &{\tiny $\pm$\tt{2}} & \tt{85} &{\tiny $\pm$\tt{5}} \\
 & \tt{task3} & \tt{2} &{\tiny $\pm$\tt{3}} & \tt{14} &{\tiny $\pm$\tt{6}} & \tt{82} &{\tiny $\pm$\tt{8}} & \tt{87} &{\tiny $\pm$\tt{5}} & \tt{49} &{\tiny $\pm$\tt{10}} & \tt{13} &{\tiny $\pm$\tt{6}} & \tt{1} &{\tiny $\pm$\tt{2}} & \tt{4} &{\tiny $\pm$\tt{3}} & \tt{89} &{\tiny $\pm$\tt{5}} \\
 & \tt{task4} & \tt{2} &{\tiny $\pm$\tt{1}} & \tt{11} &{\tiny $\pm$\tt{7}} & \tt{34} &{\tiny $\pm$\tt{6}} & \tt{59} &{\tiny $\pm$\tt{2}} & \tt{13} &{\tiny $\pm$\tt{7}} & \tt{3} &{\tiny $\pm$\tt{3}} & \tt{3} &{\tiny $\pm$\tt{1}} & \tt{2} &{\tiny $\pm$\tt{1}} & \tt{48} &{\tiny $\pm$\tt{12}} \\
 & \tt{task5} & \tt{1} &{\tiny $\pm$\tt{1}} & \tt{6} &{\tiny $\pm$\tt{3}} & \tt{54} &{\tiny $\pm$\tt{4}} & \tt{63} &{\tiny $\pm$\tt{7}} & \tt{22} &{\tiny $\pm$\tt{8}} & \tt{8} &{\tiny $\pm$\tt{3}} & \tt{3} &{\tiny $\pm$\tt{1}} & \tt{3} &{\tiny $\pm$\tt{1}} & \tt{53} &{\tiny $\pm$\tt{3}} \\
 & \tt{overall} & \tt{3} &{\tiny $\pm$\tt{1}} & \tt{17} &{\tiny $\pm$\tt{2}} & \tt{62} &{\tiny $\pm$\tt{2}} & \tt{77} &{\tiny $\pm$\tt{3}} & \tt{35} &{\tiny $\pm$\tt{5}} & \tt{13} &{\tiny $\pm$\tt{1}} & \tt{5} &{\tiny $\pm$\tt{1}} & \tt{5} &{\tiny $\pm$\tt{2}} & \tt{73} &{\tiny $\pm$\tt{4}} \\
\cmidrule{1-20}
\multirow[c]{6}{*}{\tt{cube-single-play-oraclerep-v0}} & \tt{task1} & \tt{7} &{\tiny $\pm$\tt{9}} & \tt{17} &{\tiny $\pm$\tt{5}} & \tt{87} &{\tiny $\pm$\tt{3}} & \tt{97} &{\tiny $\pm$\tt{3}} & \tt{64} &{\tiny $\pm$\tt{6}} & \tt{6} &{\tiny $\pm$\tt{7}} & \tt{4} &{\tiny $\pm$\tt{3}} & \tt{7} &{\tiny $\pm$\tt{7}} & \tt{98} &{\tiny $\pm$\tt{2}} \\
 & \tt{task2} & \tt{8} &{\tiny $\pm$\tt{6}} & \tt{18} &{\tiny $\pm$\tt{12}} & \tt{93} &{\tiny $\pm$\tt{4}} & \tt{96} &{\tiny $\pm$\tt{3}} & \tt{61} &{\tiny $\pm$\tt{8}} & \tt{8} &{\tiny $\pm$\tt{4}} & \tt{5} &{\tiny $\pm$\tt{4}} & \tt{6} &{\tiny $\pm$\tt{12}} & \tt{97} &{\tiny $\pm$\tt{3}} \\
 & \tt{task3} & \tt{9} &{\tiny $\pm$\tt{6}} & \tt{22} &{\tiny $\pm$\tt{5}} & \tt{93} &{\tiny $\pm$\tt{2}} & \tt{99} &{\tiny $\pm$\tt{1}} & \tt{69} &{\tiny $\pm$\tt{9}} & \tt{9} &{\tiny $\pm$\tt{3}} & \tt{1} &{\tiny $\pm$\tt{2}} & \tt{2} &{\tiny $\pm$\tt{3}} & \tt{99} &{\tiny $\pm$\tt{1}} \\
 & \tt{task4} & \tt{6} &{\tiny $\pm$\tt{2}} & \tt{20} &{\tiny $\pm$\tt{6}} & \tt{85} &{\tiny $\pm$\tt{7}} & \tt{92} &{\tiny $\pm$\tt{4}} & \tt{56} &{\tiny $\pm$\tt{10}} & \tt{2} &{\tiny $\pm$\tt{2}} & \tt{3} &{\tiny $\pm$\tt{2}} & \tt{27} &{\tiny $\pm$\tt{17}} & \tt{93} &{\tiny $\pm$\tt{6}} \\
 & \tt{task5} & \tt{5} &{\tiny $\pm$\tt{5}} & \tt{14} &{\tiny $\pm$\tt{7}} & \tt{82} &{\tiny $\pm$\tt{4}} & \tt{89} &{\tiny $\pm$\tt{5}} & \tt{66} &{\tiny $\pm$\tt{18}} & \tt{3} &{\tiny $\pm$\tt{2}} & \tt{0} &{\tiny $\pm$\tt{0}} & \tt{24} &{\tiny $\pm$\tt{21}} & \tt{87} &{\tiny $\pm$\tt{7}} \\
 & \tt{overall} & \tt{7} &{\tiny $\pm$\tt{2}} & \tt{18} &{\tiny $\pm$\tt{5}} & \tt{88} &{\tiny $\pm$\tt{2}} & \tt{95} &{\tiny $\pm$\tt{1}} & \tt{63} &{\tiny $\pm$\tt{5}} & \tt{6} &{\tiny $\pm$\tt{3}} & \tt{3} &{\tiny $\pm$\tt{1}} & \tt{13} &{\tiny $\pm$\tt{8}} & \tt{95} &{\tiny $\pm$\tt{2}} \\
\cmidrule{1-20}
\multirow[c]{6}{*}{\tt{cube-double-play-oraclerep-v0}} & \tt{task1} & \tt{6} &{\tiny $\pm$\tt{5}} & \tt{17} &{\tiny $\pm$\tt{8}} & \tt{88} &{\tiny $\pm$\tt{4}} & \tt{92} &{\tiny $\pm$\tt{6}} & \tt{77} &{\tiny $\pm$\tt{8}} & \tt{7} &{\tiny $\pm$\tt{4}} & \tt{6} &{\tiny $\pm$\tt{2}} & \tt{1} &{\tiny $\pm$\tt{2}} & \tt{73} &{\tiny $\pm$\tt{5}} \\
 & \tt{task2} & \tt{0} &{\tiny $\pm$\tt{0}} & \tt{1} &{\tiny $\pm$\tt{1}} & \tt{78} &{\tiny $\pm$\tt{5}} & \tt{84} &{\tiny $\pm$\tt{7}} & \tt{42} &{\tiny $\pm$\tt{9}} & \tt{0} &{\tiny $\pm$\tt{0}} & \tt{0} &{\tiny $\pm$\tt{0}} & \tt{0} &{\tiny $\pm$\tt{0}} & \tt{23} &{\tiny $\pm$\tt{7}} \\
 & \tt{task3} & \tt{0} &{\tiny $\pm$\tt{0}} & \tt{0} &{\tiny $\pm$\tt{0}} & \tt{75} &{\tiny $\pm$\tt{6}} & \tt{85} &{\tiny $\pm$\tt{2}} & \tt{39} &{\tiny $\pm$\tt{10}} & \tt{0} &{\tiny $\pm$\tt{0}} & \tt{0} &{\tiny $\pm$\tt{0}} & \tt{0} &{\tiny $\pm$\tt{0}} & \tt{30} &{\tiny $\pm$\tt{11}} \\
 & \tt{task4} & \tt{0} &{\tiny $\pm$\tt{0}} & \tt{1} &{\tiny $\pm$\tt{2}} & \tt{8} &{\tiny $\pm$\tt{5}} & \tt{12} &{\tiny $\pm$\tt{9}} & \tt{1} &{\tiny $\pm$\tt{1}} & \tt{0} &{\tiny $\pm$\tt{0}} & \tt{0} &{\tiny $\pm$\tt{0}} & \tt{0} &{\tiny $\pm$\tt{0}} & \tt{3} &{\tiny $\pm$\tt{3}} \\
 & \tt{task5} & \tt{0} &{\tiny $\pm$\tt{0}} & \tt{2} &{\tiny $\pm$\tt{2}} & \tt{47} &{\tiny $\pm$\tt{14}} & \tt{45} &{\tiny $\pm$\tt{8}} & \tt{17} &{\tiny $\pm$\tt{4}} & \tt{0} &{\tiny $\pm$\tt{0}} & \tt{0} &{\tiny $\pm$\tt{0}} & \tt{0} &{\tiny $\pm$\tt{0}} & \tt{18} &{\tiny $\pm$\tt{7}} \\
 & \tt{overall} & \tt{1} &{\tiny $\pm$\tt{1}} & \tt{4} &{\tiny $\pm$\tt{2}} & \tt{59} &{\tiny $\pm$\tt{2}} & \tt{64} &{\tiny $\pm$\tt{4}} & \tt{35} &{\tiny $\pm$\tt{3}} & \tt{1} &{\tiny $\pm$\tt{1}} & \tt{1} &{\tiny $\pm$\tt{0}} & \tt{0} &{\tiny $\pm$\tt{0}} & \tt{30} &{\tiny $\pm$\tt{5}} \\
\cmidrule{1-20}
\multirow[c]{6}{*}{\tt{scene-play-oraclerep-v0}} & \tt{task1} & \tt{14} &{\tiny $\pm$\tt{8}} & \tt{60} &{\tiny $\pm$\tt{6}} & \tt{97} &{\tiny $\pm$\tt{3}} & \tt{99} &{\tiny $\pm$\tt{1}} & \tt{71} &{\tiny $\pm$\tt{13}} & \tt{19} &{\tiny $\pm$\tt{8}} & \tt{43} &{\tiny $\pm$\tt{4}} & \tt{32} &{\tiny $\pm$\tt{8}} & \tt{97} &{\tiny $\pm$\tt{2}} \\
 & \tt{task2} & \tt{2} &{\tiny $\pm$\tt{1}} & \tt{13} &{\tiny $\pm$\tt{3}} & \tt{92} &{\tiny $\pm$\tt{4}} & \tt{96} &{\tiny $\pm$\tt{2}} & \tt{14} &{\tiny $\pm$\tt{4}} & \tt{2} &{\tiny $\pm$\tt{2}} & \tt{9} &{\tiny $\pm$\tt{2}} & \tt{2} &{\tiny $\pm$\tt{2}} & \tt{95} &{\tiny $\pm$\tt{3}} \\
 & \tt{task3} & \tt{2} &{\tiny $\pm$\tt{3}} & \tt{21} &{\tiny $\pm$\tt{9}} & \tt{83} &{\tiny $\pm$\tt{8}} & \tt{89} &{\tiny $\pm$\tt{8}} & \tt{33} &{\tiny $\pm$\tt{9}} & \tt{1} &{\tiny $\pm$\tt{1}} & \tt{3} &{\tiny $\pm$\tt{4}} & \tt{2} &{\tiny $\pm$\tt{3}} & \tt{97} &{\tiny $\pm$\tt{3}} \\
 & \tt{task4} & \tt{3} &{\tiny $\pm$\tt{1}} & \tt{19} &{\tiny $\pm$\tt{8}} & \tt{43} &{\tiny $\pm$\tt{28}} & \tt{13} &{\tiny $\pm$\tt{11}} & \tt{12} &{\tiny $\pm$\tt{3}} & \tt{6} &{\tiny $\pm$\tt{3}} & \tt{2} &{\tiny $\pm$\tt{2}} & \tt{3} &{\tiny $\pm$\tt{1}} & \tt{76} &{\tiny $\pm$\tt{17}} \\
 & \tt{task5} & \tt{0} &{\tiny $\pm$\tt{0}} & \tt{3} &{\tiny $\pm$\tt{4}} & \tt{23} &{\tiny $\pm$\tt{16}} & \tt{10} &{\tiny $\pm$\tt{9}} & \tt{2} &{\tiny $\pm$\tt{2}} & \tt{1} &{\tiny $\pm$\tt{1}} & \tt{0} &{\tiny $\pm$\tt{0}} & \tt{0} &{\tiny $\pm$\tt{0}} & \tt{18} &{\tiny $\pm$\tt{9}} \\
 & \tt{overall} & \tt{4} &{\tiny $\pm$\tt{2}} & \tt{23} &{\tiny $\pm$\tt{2}} & \tt{68} &{\tiny $\pm$\tt{10}} & \tt{61} &{\tiny $\pm$\tt{3}} & \tt{26} &{\tiny $\pm$\tt{5}} & \tt{6} &{\tiny $\pm$\tt{2}} & \tt{12} &{\tiny $\pm$\tt{1}} & \tt{8} &{\tiny $\pm$\tt{2}} & \tt{77} &{\tiny $\pm$\tt{2}} \\
\cmidrule{1-20}
\multirow[c]{6}{*}{\tt{puzzle-3x3-play-oraclerep-v0}} & \tt{task1} & \tt{4} &{\tiny $\pm$\tt{4}} & \tt{10} &{\tiny $\pm$\tt{1}} & \tt{4} &{\tiny $\pm$\tt{2}} & \tt{99} &{\tiny $\pm$\tt{1}} & \tt{15} &{\tiny $\pm$\tt{8}} & \tt{3} &{\tiny $\pm$\tt{3}} & \tt{5} &{\tiny $\pm$\tt{1}} & \tt{7} &{\tiny $\pm$\tt{3}} & \tt{99} &{\tiny $\pm$\tt{1}} \\
 & \tt{task2} & \tt{1} &{\tiny $\pm$\tt{2}} & \tt{1} &{\tiny $\pm$\tt{1}} & \tt{2} &{\tiny $\pm$\tt{2}} & \tt{99} &{\tiny $\pm$\tt{1}} & \tt{6} &{\tiny $\pm$\tt{3}} & \tt{0} &{\tiny $\pm$\tt{0}} & \tt{1} &{\tiny $\pm$\tt{1}} & \tt{2} &{\tiny $\pm$\tt{2}} & \tt{99} &{\tiny $\pm$\tt{1}} \\
 & \tt{task3} & \tt{1} &{\tiny $\pm$\tt{1}} & \tt{1} &{\tiny $\pm$\tt{1}} & \tt{2} &{\tiny $\pm$\tt{1}} & \tt{99} &{\tiny $\pm$\tt{1}} & \tt{1} &{\tiny $\pm$\tt{1}} & \tt{0} &{\tiny $\pm$\tt{0}} & \tt{0} &{\tiny $\pm$\tt{0}} & \tt{1} &{\tiny $\pm$\tt{2}} & \tt{100} &{\tiny $\pm$\tt{0}} \\
 & \tt{task4} & \tt{0} &{\tiny $\pm$\tt{0}} & \tt{1} &{\tiny $\pm$\tt{1}} & \tt{1} &{\tiny $\pm$\tt{1}} & \tt{98} &{\tiny $\pm$\tt{2}} & \tt{1} &{\tiny $\pm$\tt{1}} & \tt{0} &{\tiny $\pm$\tt{0}} & \tt{2} &{\tiny $\pm$\tt{3}} & \tt{1} &{\tiny $\pm$\tt{1}} & \tt{98} &{\tiny $\pm$\tt{1}} \\
 & \tt{task5} & \tt{1} &{\tiny $\pm$\tt{1}} & \tt{2} &{\tiny $\pm$\tt{2}} & \tt{1} &{\tiny $\pm$\tt{1}} & \tt{95} &{\tiny $\pm$\tt{2}} & \tt{2} &{\tiny $\pm$\tt{2}} & \tt{0} &{\tiny $\pm$\tt{0}} & \tt{1} &{\tiny $\pm$\tt{1}} & \tt{0} &{\tiny $\pm$\tt{0}} & \tt{99} &{\tiny $\pm$\tt{1}} \\
 & \tt{overall} & \tt{1} &{\tiny $\pm$\tt{1}} & \tt{3} &{\tiny $\pm$\tt{1}} & \tt{2} &{\tiny $\pm$\tt{1}} & \tt{98} &{\tiny $\pm$\tt{0}} & \tt{5} &{\tiny $\pm$\tt{1}} & \tt{1} &{\tiny $\pm$\tt{1}} & \tt{2} &{\tiny $\pm$\tt{0}} & \tt{2} &{\tiny $\pm$\tt{0}} & \tt{99} &{\tiny $\pm$\tt{0}} \\
\cmidrule{1-20}
\multirow[c]{6}{*}{\tt{puzzle-4x4-play-oraclerep-v0}} & \tt{task1} & \tt{1} &{\tiny $\pm$\tt{1}} & \tt{1} &{\tiny $\pm$\tt{1}} & \tt{6} &{\tiny $\pm$\tt{4}} & \tt{33} &{\tiny $\pm$\tt{14}} & \tt{1} &{\tiny $\pm$\tt{1}} & \tt{0} &{\tiny $\pm$\tt{0}} & \tt{1} &{\tiny $\pm$\tt{1}} & \tt{0} &{\tiny $\pm$\tt{0}} & \tt{47} &{\tiny $\pm$\tt{5}} \\
 & \tt{task2} & \tt{0} &{\tiny $\pm$\tt{0}} & \tt{1} &{\tiny $\pm$\tt{1}} & \tt{4} &{\tiny $\pm$\tt{4}} & \tt{0} &{\tiny $\pm$\tt{0}} & \tt{0} &{\tiny $\pm$\tt{0}} & \tt{0} &{\tiny $\pm$\tt{0}} & \tt{0} &{\tiny $\pm$\tt{0}} & \tt{1} &{\tiny $\pm$\tt{1}} & \tt{17} &{\tiny $\pm$\tt{5}} \\
 & \tt{task3} & \tt{0} &{\tiny $\pm$\tt{0}} & \tt{1} &{\tiny $\pm$\tt{1}} & \tt{6} &{\tiny $\pm$\tt{1}} & \tt{58} &{\tiny $\pm$\tt{10}} & \tt{1} &{\tiny $\pm$\tt{1}} & \tt{0} &{\tiny $\pm$\tt{0}} & \tt{1} &{\tiny $\pm$\tt{1}} & \tt{0} &{\tiny $\pm$\tt{0}} & \tt{38} &{\tiny $\pm$\tt{13}} \\
 & \tt{task4} & \tt{1} &{\tiny $\pm$\tt{1}} & \tt{1} &{\tiny $\pm$\tt{1}} & \tt{6} &{\tiny $\pm$\tt{6}} & \tt{22} &{\tiny $\pm$\tt{5}} & \tt{1} &{\tiny $\pm$\tt{1}} & \tt{0} &{\tiny $\pm$\tt{0}} & \tt{0} &{\tiny $\pm$\tt{0}} & \tt{0} &{\tiny $\pm$\tt{0}} & \tt{34} &{\tiny $\pm$\tt{2}} \\
 & \tt{task5} & \tt{0} &{\tiny $\pm$\tt{0}} & \tt{0} &{\tiny $\pm$\tt{0}} & \tt{6} &{\tiny $\pm$\tt{1}} & \tt{29} &{\tiny $\pm$\tt{9}} & \tt{0} &{\tiny $\pm$\tt{0}} & \tt{0} &{\tiny $\pm$\tt{0}} & \tt{1} &{\tiny $\pm$\tt{1}} & \tt{0} &{\tiny $\pm$\tt{0}} & \tt{32} &{\tiny $\pm$\tt{6}} \\
 & \tt{overall} & \tt{0} &{\tiny $\pm$\tt{0}} & \tt{1} &{\tiny $\pm$\tt{0}} & \tt{5} &{\tiny $\pm$\tt{2}} & \tt{28} &{\tiny $\pm$\tt{4}} & \tt{0} &{\tiny $\pm$\tt{0}} & \tt{0} &{\tiny $\pm$\tt{0}} & \tt{0} &{\tiny $\pm$\tt{0}} & \tt{0} &{\tiny $\pm$\tt{0}} & \tt{34} &{\tiny $\pm$\tt{4}} \\
\bottomrule
\end{tabularew}

}
\end{table}

\clearpage

\section{Experimental Details}
\label{sec:exp_details}

In this section, we describe the full details of our experiments.
We provide the code and instructions at \url{https://github.com/aoberai/trl}.

\subsection{Methods}
\label{sec:methods}

Here, we describe the previous offline GCRL algorithms considered in this work.

\textbf{BC, FBC, IVL~\citep{iql_kostrikov2022, hiql_park2023}, IQL~\citep{iql_kostrikov2022}, CRL~\citep{crl_eysenbach2022}, and QRL~\citep{qrl_wang2023}.}
We employ the original implementations by \citet{ogbench_park2025, sharsa_park2025},
and we refer to these works for the full details.

\textbf{TD and TD-$\bm{n}$.}
The TD (which is equivalent to TD-$1$) and TD-$n$ baselines used in \Cref{sec:exp_long} minimize the following IQL-like objective:
\begin{align}
    L^0(Q) &= \E_{\tau \sim \gD}\left[D\left(Q(s_i, a_i, s_i), \gamma^0\right)\right], \\
    L^1(Q) &= \E_{\tau \sim \gD}\left[D_\kappa \left(Q(s_i, a_i, s_j), \gamma^n \bar Q(s_{i+n}, a_{i+n}, s_j)\right)\right], \\
    L^{\mathrm{TD-}n}(Q) &= L^0(Q) + L^1(Q),
\end{align}
where trajectories $\tau$ are sampled from the same trajectory distribution used in \methodname.
When $i+n > j$, we appropriately clip the value of $n$ (this is omitted in the objective above for notational simplicity).
Intuitively, the objective above can be viewed as the closest TD-based variant (ablation) of \methodname.

\textbf{TDP~\citep{gcrl_kaelbling1993, fwrl_dhiman2018, sgt_jurgenson2020}.}
We refer to triangle-inequality dynamic programming (TDP) as GCRL algorithms
that take a hard maximum over the entire state space to compute an optimal subgoal~\citep{gcrl_kaelbling1993, fwrl_dhiman2018, sgt_jurgenson2020}.
These algorithms were originally designed for tabular environments,
so we consider a sampling-based variant of TDP for our experiments.
Specifically, to compute a hard maximum in our continuous-control environments,
we sample $M$ subgoal candidates from the dataset,
and take the maximum over these $M$ states.
Formally, we minimize the following loss:
\begin{align}
    L^0(Q) &= \E_{s, a \sim \gD}\left[D\left( Q(s, a, s), \gamma^0 \right)\right], \\
    L^1(Q) &= \E_{s, a, s' \sim \gD}\left[D\left( Q(s, a, s'), \gamma^1 \right)\right], \\
    L^\infty(Q) &= \E_{s, a, g^r \sim \gD}\left[D\left( Q(s, a, g^r), \gamma^P \right)\right], \\
    L^\mathrm{\triangle}(Q) &= \E_{s, a, g, W \sim \gD}\left[D\left(
        Q(s, a, g), \max_{(w, a_w) \in W} \bar Q(s, a, w) \bar Q(w, a_w, g)
    \right)\right], \\
    L^\mathrm{TDP}(Q) &= L^0(Q) + L^1(Q) + L^\infty(Q) + L^\mathrm{\triangle}(Q),
\end{align}
where $g^r$ is a randomly sampled goal from the dataset
(note that this distribution is different from that of $g$, as $g$ is typically partially sampled with hindsight relabeling),
$P$ is a tunable hyperparameter (a large number to approximate the distance between a random $(s, g)$ pair~\citep{sgt_jurgenson2020}),
and $W = \{(w^{(i)}, a_w^{(i)})\}_{i=1}^M$ is a set of $M$ randomly sampled state-action pairs from the dataset.

\textbf{COE~\citep{coe_piekos2023}.}
COE is a triangle inequality-based GCRL algorithm that uses a separate generator network
$G(\pl{s}, \pl{a}, \pl{g}): \gS \times \gA \times \gS \to \gS$
to predict the optimal subgoal.
Since this algorithm was originally designed for online GCRL,
we add an additional behavioral regularizer to constrain this generator network to produce in-distribution states,
as commonly done in offline RL~\citep{brac_wu2019, td3bc_fujimoto2021, rebrac_tarasov2023}.
Formally, we minimize the following losses to train the value function and generator:
\begin{align}
    L^1(Q) &= \E_{s, a, s' \sim \gD}\left[D\left( Q(s, a, s'), \gamma^1 \right)\right], \\
    L^\mathrm{\triangle}(Q) &= \E_{\substack{s, a, g \sim \gD, \\ w = G(s, a, g)}}\left[D\left(
        Q(s, a, g), \bar Q(s, a, w) \bar Q(w, \pi(w, g), g)
    \right)\right], \\
    L^\mathrm{COE}(Q) &= L^1(Q) + L^\mathrm{\triangle}(Q), \\
    L^\mathrm{COE}(G) &= \E_{\substack{s, a, g, g^r \sim \gD, \\ w = G(s, a, g)}}\left[
        - \bar Q(s, a, w) \bar Q(w, \pi(w, g), g) - \beta \|w - g^r\|_2^2
    \right],
\end{align}
where $g^r$ is a randomly sampled goal from the dataset,
$\beta$ is a hyperparameter that controls the strength of the regularizer.
In the above, we slightly abuse the notation by assuming that the goal-conditioned policy $\pi$ is deterministic.

\subsection{Oracle Distillation}
In our experiments, we mainly employ the \tt{oraclerep} variants, following \citet{sharsa_park2025}.
In \tt{oraclerep} environments, we are given an ``oracle'' representation of goals
($\phi(\pl{g}): \gS \to \gZ$, where $\gZ$ is an oracle goal representation space),
which typically corresponds to a subset of the state dimensions.
For example, in \tt{humanoidmaze},
a full (goal) state is a $69$-dimensional vector,
but the corresponding oracle representation is a $2$-dimensional vector consisting only of the $x$-$y$ coordinates.
In these environments, we condition the policy and value functions on the oracle representation, not on the full goal state.

While we can easily incorporate this change in standard MC- and TD-based algorithms
by simply parameterizing the policy and value functions with the oracle representation
(\eg, using $\pi(\pl{a} \mid \pl{s}, \phi(\pl{g}))$ instead of $\pi(\pl{a} \mid \pl{s}, \pl{g})$),
it is not straightforward to use $\phi$ in triangle inequality-based methods,
such as QRL and \methodname,
because they assume that states and goals lie in the same space.
To make such methods compatible with oracle representations, we apply a technique we call ``oracle distillation.''
That is, we train an additional oracle representation-conditioned Q function, $Q^\phi: \gS \times \gA \times \gZ \to \sR$,
by distilling it from the original Q function, $Q: \gS \times \gA \times \gS \to \sR$, using the following loss:
\begin{align}
    L^\mathrm{distill}(Q^\phi) = \E_{s, a, g \sim \gD}\left[D\left(Q^\phi(s, a, \phi(g)), Q(s, a, g)\right)\right].
\end{align}
Here, the original Q function assumes that the state and goal spaces are the same and can therefore leverage the triangle inequality.
After training $Q^\phi$, we extract an oracle representation-conditioned policy from this distilled Q function.

\subsection{Implementation Details}

\textbf{Training and evaluation.}
In our experiments, we train all agents for $1$M gradient steps.
For evaluation, we use $15$ episodes for each of the five evaluation goals provided by OGBench~\citep{ogbench_park2025}.
In tables, we report performance averaged over the last three evaluation epochs
(\ie, $800$K, $900$K, and $1$M steps),
following the protocol of OGBench.
However, for the baselines from the work by \citet{sharsa_park2025} (\ie, FBC, IQL, CRL, and SAC+BC in \Cref{table:1b}),
we use the evaluation result only at the $1$M epoch, as they were evaluated every $500$K steps.

In \Cref{table:1b}, we take the performance of FBC, IQL, CRL, and SAC+BC from the work by \citet{sharsa_park2025}.
In \Cref{table:standard}, we re-run all baselines (BC, FBC, IVL, IQL, CRL, and QRL)
using the original OGBench hyperparameters~\citep{ogbench_park2025}
on the \tt{oraclerep} tasks
(except for the value goal hindsight relabeling ratios, which we describe below).
In these tables,
we ensure apples-to-apples comparisons between \methodname and other baselines by using the same default configurations
(\eg, we use the same network size, discount factor, training steps, etc.).

\textbf{Hyperparameters.}
We provide the full list of hyperparameters in \Cref{table:hyp_1b,table:hyp_standard,table:hyp_task}.
For the experiments in \Cref{sec:exp_long},
we mostly follow the hyperparameters by \citet{sharsa_park2025},
and for the experiments in \Cref{sec:exp_standard},
we mostly follow the ones by \citet{ogbench_park2025}.
In the hyperparameter tables,
the tuple $(p^\gD_\mathrm{cur}, p^\gD_\mathrm{geom}, p^\gD_\mathrm{traj}, p^\gD_\mathrm{rand})$
denotes the hindsight goal relabeling ratios described in OGBench~\citep{ogbench_park2025}.

\clearpage

\begin{table}[h!]
\caption{
\footnotesize
\textbf{Hyperparameters for long-horizon OGBench tasks (\Cref{table:1b}).}
}
\vspace{-5pt}
\label{table:hyp_1b}
\begin{center}
\scalebox{0.78}
{
\begin{tabular}{ll}
    \toprule
    \textbf{Hyperparameter} & \textbf{Value} \\
    \midrule
    Gradient steps & $10^6$ \\
    Optimizer & Adam~\citep{adam_kingma2015} \\
    Learning rate & $0.0003$ \\
    Batch size & $1024$ \\
    MLP size & $[1024, 1024, 1024, 1024]$ \\
    Nonlinearity & GELU~\citep{gelu_hendrycks2016} \\
    Target network update rate & $0.005$ \\
    Discount factor $\gamma$ & $0.999$ \\
    Policy $(p^\gD_\mathrm{cur}, p^\gD_\mathrm{geom}, p^\gD_\mathrm{traj}, p^\gD_\mathrm{rand})$ ratio & $(0, 0, 1, 0)$ (\tt{humanoidmaze}), $(0, 0.5, 0, 0.5)$ (\tt{puzzle}) \\
    Value $(p^\gD_\mathrm{cur}, p^\gD_\mathrm{geom}, p^\gD_\mathrm{traj}, p^\gD_\mathrm{rand})$ ratio (TRL, TD, MC) & $(0, 0, 1, 0)$ \\
    Value $(p^\gD_\mathrm{cur}, p^\gD_\mathrm{geom}, p^\gD_\mathrm{traj}, p^\gD_\mathrm{rand})$ ratio (CRL) & $(0, 1, 0, 0)$ \\
    Value $(p^\gD_\mathrm{cur}, p^\gD_\mathrm{geom}, p^\gD_\mathrm{traj}, p^\gD_\mathrm{rand})$ ratio (others) & $(0.2, 0.5, 0, 0.3)$ \\
    Policy extraction & Reparameterized gradients or rejection sampling (see \Cref{table:hyp_task}) \\
    \bottomrule
\end{tabular}
}
\end{center}
\end{table}

\begin{table}[h!]
\caption{
\footnotesize
\textbf{Hyperparameters for standard OGBench tasks (\Cref{table:standard}).}
}
\vspace{-5pt}
\label{table:hyp_standard}
\begin{center}
\scalebox{0.78}
{
\begin{tabular}{ll}
    \toprule
    \textbf{Hyperparameter} & \textbf{Value} \\
    \midrule
    Gradient steps & $10^6$ \\
    Optimizer & Adam~\citep{adam_kingma2015} \\
    Learning rate & $0.0003$ \\
    Batch size & $1024$ \\
    MLP size & $[512, 512, 512]$ \\
    Nonlinearity & GELU~\citep{gelu_hendrycks2016} \\
    Target network update rate & $0.005$ \\
    Discount factor $\gamma$ & $0.99$ (default), $0.995$ (\tt{humanoidmaze}) \\
    Policy $(p^\gD_\mathrm{cur}, p^\gD_\mathrm{geom}, p^\gD_\mathrm{traj}, p^\gD_\mathrm{rand})$ ratio & $(0, 0, 1, 0)$ \\
    Value $(p^\gD_\mathrm{cur}, p^\gD_\mathrm{geom}, p^\gD_\mathrm{traj}, p^\gD_\mathrm{rand})$ ratio (\methodname, CRL) & $(0, 1, 0, 0)$ \\
    Value $(p^\gD_\mathrm{cur}, p^\gD_\mathrm{geom}, p^\gD_\mathrm{traj}, p^\gD_\mathrm{rand})$ ratio (others) & $(0.2, 0.5, 0, 0.3)$ \\
    Policy extraction & Reparameterized gradients (see \Cref{table:hyp_task}) \\
    \bottomrule
\end{tabular}
}
\end{center}
\end{table}

\begin{table}[h!]
\caption{
\footnotesize
\textbf{Task-specific hyperparameters.} 
We describe task-specific hyperparameters below
($\alpha$: BC coefficient for reparameterized gradients,
$N$: sample count for rejection sampling,
$M$: subgoal count,
$P$: random goal distance,
$\beta$: goal regularization coefficient,
$\kappa$: expectile,
$\lambda$: distance-based re-weighting factor).
}
\vspace{-5pt}
\label{table:hyp_task}
\begin{center}
\scalebox{0.78}
{
\begin{tabular}{lccccccccccc}
\toprule
 & \tt{QRL} & \tt{TDP} & \tt{TDP} & \tt{TDP} & \tt{COE} & \tt{COE} & \tt{TD} & \tt{MC} & \tt{TRL} & \tt{TRL} & \tt{TRL} \\
\tt{Environment} & $\alpha$ & $(\alpha, N)$ & $M$ & $P$ & $\alpha$ & $\beta$ & $(\alpha, N)$ & $(\alpha, N)$ & $(\alpha, N)$ & $\kappa$ & $\lambda$ \\
\midrule
\tt{humanoidmaze-giant} & $0.001$ & $(3, -)$ & $8$ & $1000$ & $1$ & $0.5$ & $(0.3, -)$ & $(0.3, -)$ & $(0.1, -)$ & $0.7$ & $0$ \\
\tt{puzzle-4x5} & $3$ & $(-, 32)$ & $8$ & $1000$ & $3$ & $10$ & $(-, 32)$ & $(-, 32)$ & $(-, 32)$ & $0.7$ & $0$ \\
\tt{puzzle-4x6} & $3$ & $(-, 32)$ & $8$ & $1000$ & $3$ & $10$ & $(-, 32)$ & $(-, 32)$ & $(-, 32)$ & $0.7$ & $0$ \\
\midrule
\tt{pointmaze-large} & $0.0003$ & $(5, -)$ & $8$ & $500$ & $1$ & $1$ &  &  & $(10, -)$ & $0.7$ & $0.7$ \\
\tt{antmaze-large} & $0.003$ & $(10, -)$ & $8$ & $100$ & $3$ & $10$ &  &  & $(0.7, -)$ & $0.7$ & $0$ \\
\tt{humanoidmaze-medium} & $0.001$ & $(5, -)$ & $8$ & $500$ & $0.1$ & $3$ &  &  & $(0.1, -)$ & $0.7$ & $0$ \\
\tt{humanoidmaze-large} & $0.001$ & $(5, -)$ & $8$ & $500$ & $0.1$ & $1$ &  &  & $(0.1, -)$ & $0.7$ & $0.1$ \\
\tt{antsoccer-arena} & $0.003$ & $(10, -)$ & $8$ & $200$ & $0.3$ & $1$ &  &  & $(0.3, -)$ & $0.7$ & $0.5$ \\
\tt{cube-single} & $0.3$ & $(5, -)$ & $8$ & $500$ & $0.3$ & $10$ &  &  & $(1, -)$ & $0.7$ & $0.7$ \\
\tt{cube-double} & $0.3$ & $(5, -)$ & $8$ & $500$ & $0.3$ & $10$ &  &  & $(10, -)$ & $0.7$ & $1$ \\
\tt{scene} & $0.3$ & $(1, -)$ & $16$ & $200$ & $1$ & $10$ &  &  & $(1, -)$ & $0.7$ & $1$ \\
\tt{puzzle-3x3} & $0.3$ & $(5, -)$ & $8$ & $500$ & $3$ & $10$ &  &  & $(2, -)$ & $0.7$ & $0.5$ \\
\tt{puzzle-4x4} & $0.3$ & $(5, -)$ & $8$ & $500$ & $1$ & $10$ &  &  & $(2, -)$ & $0.7$ & $2$ \\
\bottomrule
\end{tabular}
}
\end{center}
\end{table}

\end{document}